\documentclass[lettersize,journal]{IEEEtran}
\usepackage{amsmath,amsfonts}
\usepackage{algorithmic}
\usepackage{algorithm}
\usepackage{array}
\usepackage[caption=false,font=normalsize,labelfont=sf,textfont=sf]{subfig}
\usepackage{textcomp}
\usepackage{stfloats}
\usepackage{url}
\usepackage{verbatim}
\usepackage{graphicx}

\usepackage{amsthm}
\usepackage{thmtools}
\usepackage{thm-restate}

\usepackage{cite}
\usepackage{orcidlink}
\theoremstyle{plain}
\newtheorem{theorem}{Theorem}
\newtheorem{proposition}{Proposition}
\newtheorem{lemma}{Lemma}

\newtheorem{definition}{Definition}
\newtheorem{remark}{Remark}

\newtheorem{theoremR}{Theorem}
\newcommand{\theoremRestate}[1]{%
  \renewcommand{\thetheoremR}{\ref{#1}}%
}

\newtheorem{propT}{Proposition}
\makeatletter
\newcommand{\propRestate}[1]{%
  \expandafter\def\csname thepropT\endcsname{\ref{#1}}%
}
\makeatother

\begin{document}

\title{Group-Invariant Unsupervised Skill Discovery: Symmetry-aware Skill Representations for Generalizable Behavior}

\author{%
Junwoo~Chang\orcidlink{0009-0008-5843-6890},
Joseph~Park\orcidlink{0009-0000-0765-0450},
Roberto~Horowitz\orcidlink{0000-0003-2807-0054},
Jongmin~Lee\orcidlink{0009-0008-0761-1329},
and~Jongeun~Choi\orcidlink{0000-0002-7532-5315}%
\thanks{(Corresponding authors: Jongeun Choi and Jongmin Lee.)}%
\thanks{Junwoo Chang, Joseph Park, and Jongeun Choi are with the School of Mechanical Engineering, 
Yonsei University, Seoul, South Korea (\texttt{e-mails: junwoochang@yonsei.ac.kr;
iamjoseph1129@yonsei.ac.kr; jongeunchoi@yonsei.ac.kr}).\\
Jongmin Lee and Jongeun Choi are with the Department of Artificial Intelligence, 
Yonsei University, Seoul, South Korea (\texttt{e-mails: jongminlee@yonsei.ac.kr;
jongeunchoi@yonsei.ac.kr}).\\
Jongeun Choi and Roberto Horowitz are with the Department of Mechanical Engineering, University of California, Berkeley, CA 94720, USA (\texttt{e-mail: horowitz@berkeley.edu})}%
\thanks{The authors used the ChatGPT large language model (OpenAI) \cite{chatgpt} to polish texts. All texts written with LLM assistance was reviewed and verified by the authors.}
\thanks{This work has been submitted to the IEEE for possible publication. Copyright may be transferred without notice, after which this version may no longer be accessible.}
}

\markboth{Preprint. Under Review.}
{Shell \MakeLowercase{\textit{et al.}}: A Sample Article Using IEEEtran.cls for IEEE Journals}


\maketitle

\begin{abstract}
Unsupervised skill discovery aims to acquire behavior primitives that improve exploration and accelerate downstream task learning. However, existing approaches often ignore the geometric symmetries of physical environments, leading to redundant behaviors and sample inefficiency. To address this, we introduce Group-Invariant Skill Discovery (GISD), a framework that explicitly embeds group structure into the skill discovery objective. Our approach is grounded in a theoretical guarantee: we prove that in group-symmetric environments, the standard Wasserstein dependency measure admits a globally optimal solution comprised of an equivariant policy and a group-invariant scoring function. Motivated by this, we formulate the Group-Invariant Wasserstein dependency measure, which restricts the optimization to this symmetry-aware subspace without loss of optimality. Practically, we parameterize the scoring function using a group Fourier representation and define the intrinsic reward via the alignment of equivariant latent features, ensuring that the discovered skills generalize systematically under group transformations. Experiments on state-based and pixel-based locomotion benchmarks demonstrate that GISD achieves broader state-space coverage and improved efficiency in downstream task learning compared to a strong baseline.
\end{abstract}

\begin{IEEEkeywords}
Skill discovery, Group equivariance, Unsupervised reinforcement learning, Robotics
\end{IEEEkeywords}

\section{Introduction}
\IEEEPARstart{U}{n}supervised skill discovery has become a central tool for acquiring reusable behavioral primitives that can accelerate downstream task learning. A wide range of methods learn latent skills by maximizing mutual information or distance-based objectives between skills and state trajectories \cite{eysenbach2018diversity, sharma2019dynamics, laskin2022cic, park2023metra, park2022lipschitz, park2023controllability}. These methods have shown that pretraining a diverse set of behaviors can substantially improve exploration, state coverage, and downstream task learning efficiency across challenging continuous-control benchmarks.

Despite this progress, existing approaches often struggle to pretrain \emph{useful} skills efficiently. Standard distance-maximizing objectives treat the state space as largely unstructured, and must discover all behaviors directly in a high-dimensional latent space. Recent works attempt to impose additional priors to mitigate this difficulty—for example, LGSD \cite{rho2024language} leverages language guidance to bias the skill space toward semantically meaningful behaviors, while ``Do's and Don'ts'' \cite{kim2024s} leverage desired behavioral preference priors. However, these methods still overlook a fundamental property of many robotic environments: the presence of strong geometric symmetries (e.g., rotations of a locomotion agent) that make many behaviors equivalent up to a group transformation. When such a structure is ignored, skill discovery tends to learn redundant variants of the same behavior, wasting samples and weakening generalization.

To address this limitation, we introduce \textbf{Group-Invariant Skill Discovery (GISD)}, a framework that injects group symmetry directly into the skill discovery objective and representation.
Our approach builds upon a key theoretical insight: in a group-invariant environment, the Wasserstein dependency measure~\cite{ozair2019wasserstein} admits a globally optimal solution comprised of an equivariant policy and a group-invariant scoring function $f$. This existence guarantee implies that we can safely restrict the search space to a class of group-invariant scoring functions without any loss of optimality. 
Based on this principle, we formulate the Group-Invariant Wasserstein dependency measure (GIWDM), which directly optimizes the objective within this reduced subspace. 
GIWDM thus eliminates the redundancy of exploring non-symmetric solutions and ensures that the discovered skills inherently respect the environmental symmetry.
Building on this theory, we parameterize the group-invariant scoring function in a group Fourier space, discovering skills whose latent representations explicitly encode the group symmetry. Once such symmetry-aware skills are learned, we can exploit the same structure: the agent can reuse a skill at novel configurations by transforming its latent representation under the group, enabling systematic generalization across group transformations.

The contributions of our work are summarized as follows:
1)We prove that, in group-symmetric MDPs, the Wasserstein Dependency Measure (WDM) admits a globally optimal solution consisting of an equivariant policy and a group-invariant scoring function $f$. This result justifies Group-Invariant Skill Discovery (GISD), which restricts $f$
to symmetry-respecting functions without sacrificing optimality.
2) We instantiate GISD by parameterizing the group-invariant scoring function $f$ in the group Fourier domain. The intrinsic reward is defined via the alignment of equivariant Fourier features, producing skill representations that are group-symmetric by construction. We also show that this symmetric representation enables generalizable skill implementation under the group transformation in downstream tasks.
3) We evaluate GISD on state-based and pixel-based locomotion benchmarks, showing improved sample efficiency and state-space coverage over a strong distance-based baseline.

\section{Related Work}
\label{sec:rrelatedwork}

\subsection{Unsupervised Skill Discovery}
Unsupervised skill discovery seeks to learn diverse behaviors in reinforcement learning (RL) without relying on explicit reward signals and reuse them for efficient downstream task training. A common strategy is to maximize the dependency between states and skills, leading to distinct trajectories according to the different latent skill vectors. One of the classes of unsupervised skill discovery contains leveraging the mutual information (MI). MI-based approaches maximize the correlation between skill and state, typically via a Kullback-Liebler (KL)-divergence formulation \cite{eysenbach2018diversity, gregor2016variational, laskin2022cic, yang2023behavior}: 
\begin{align*}
    I(\mathcal S;Z)&=D_{KL}\left(p(s,z)\| p(s)p(z)\right)\\
    &=\int_{\mathcal S \times Z} p(s,z)\log \frac{p(s,z)}{p(s)p(z)}\, dsdz.
\end{align*}
While they effectively uncover distinct skills, their reliance on KL divergence can limit overall state coverage. 
To address the limitation of MI-based approach, distance-maximizing approaches, which is the other class in unsupervised skill discovery, replace MI with the $1$-Wasserstein distance (or Kantorovich metric) for tractability, promoting both distinctness and a broader exploration of states \cite{park2022lipschitz, park2023controllability, park2023metra, rho2024language}:
\begin{equation*}
    I_{\mathcal{W}}(\mathcal S;Z)=\mathcal{W}\big(p(s,z),p(s)p(z)\big)
\end{equation*}
These approaches map states to a compact latent space that preserves a chosen distance metric (e.g., temporal distance \cite{park2023metra}, controllability distance \cite{park2023controllability}, or language distance \cite{rho2024language}), thereby promoting both broad exploration and the discovery of skills that are well-separated under the metric. 

Based on distance-maximizing unsupervised skill discovery, recent studies have proposed methods to improve the quality of skills through the addition of prior knowledge. For example, one of the prior works, LGSD \cite{rho2024language}, utilizes language prior to improve semantic diversity between skills. In another prior work, Do's and Don'ts \cite{kim2024s}, they utilize the prior knowledge of desired state transition for feasible skill discovery. 
In this work, we also build upon distance-maximizing approaches by integrating a geometric prior, but our method differs in that it discovers skills directly from the inherent group symmetric structure of the space, and thereby, the skills are enabled to be discovered with the underlying geometric symmetry.

\subsection{Group Equivariance  in Reinforcement Learning}
Group equivariance is a mathematical property that characterizes how a function's output transforms in response to a specified group action on its input. Incorporating this property into neural networks ensures that transformations of the input data induce predictable and consistent transformations of the output \cite{fuchs2020se, thomas2018tensor, geiger2022e3nn, weiler2019general, 10466590, 10045782}. Recent works have explored group equivariance in robotic settings \cite{zeng2021transporter, huang2022equivariant, huang2024fourier, ryu2022equivariant, ryu2024diffusion, kim2023robotic, simeonov2022neural, seo2025se, park2025symmetry}: for example, some prior works \cite{zeng2021transporter, huang2022equivariant, huang2024fourier} apply group-equivariant convolutions to pick-and-place tasks, and other prior works \cite{ryu2022equivariant, ryu2024diffusion, kim2023robotic, simeonov2022neural} leverage roto-translation symmetry in three-dimensional environments, notably improving data efficiency and generalization. Such benefits are particularly relevant for RL, where sample efficiency remains a central challenge. Several studies have demonstrated that integrating group equivariance into RL algorithms yields enhanced performance and faster learning compared to standard methods \cite{wang2022mathrm, wang2022equivariant, wang2022surprising, nguyen2022leveraging, tangri2024equivariant, van2020mdp, finzi2021residual, kohler2024symmetric, wang2022robot, chang2025partial}. Notably, some works \cite{wang2022mathrm, wang2022equivariant} introduce the concept of group-invariant Markov decision process (MDPs), offering a formal basis for these improvements. Building upon these successes, our work aims to incorporate group equivariance into unsupervised skill discovery, thereby addressing limitations in existing approaches and further advancing the potential of symmetry-based methods.

\section{Preliminaries}
\label{sec:prelim}
\subsection{Distance-Maximizing Skill Discovery}\label{dmsd}
Distance-maximizing skill discovery approaches \cite{park2022lipschitz, park2023controllability, park2023metra, rho2024language, atanassov2024constrained} employ the Wasserstein dependency measure (WDM) \cite{ozair2019wasserstein}, an equivalent formulation of the 1-Wasserstein distance due to the Kantorovich-Rubinstein duality \cite{villani2009optimal}. Concretely, the objective is given by:
\begin{equation*}
    \begin{aligned}
        I_{\mathcal{W}}(\mathcal S;Z)&=W_1 (p(s,z),p(s)p(z))\\
        &=\sup_{\|f\|_{L}\leq 1}  \left (\mathbb{E}_{p(s,z)}[f(s,z)]-\mathbb{E}_{p(s)p(z)}[f(s,z)] \right )
    \end{aligned}
\end{equation*}
where $W_1$ represents $1$-Wasserstein distance and $\|\cdot\|_{L}$ indicates the 1-Lipschitz constraint. The function $f$ is parameterized as $f(s,z)=\phi(s)^\top\psi(z)$, mapping the state space $\mathcal S$ to a latent space via $\phi:\mathcal S\rightarrow \mathbb{R}^D$. In addition, by imposing additional conditions, replacing $\mathcal S$ with the terminal state $S_T$, setting $\psi(z)=z$, and sampling $z$ from a zero-mean prior $p(z)$, the objective  can be approximated through a telescoping sum:

\begin{equation}\label{metra wdm}
    I_{\mathcal{W}}(\mathcal S_T;Z)\approx \sup_{\|\phi \|_{L}\leq 1}\mathbb{E}_{p(\tau,z)}\sum^{T-1}_{t=0}(\phi(s_{t+1})-\phi(s_t))^\top z
\end{equation}
Consequently, a corresponding reward function is defined as $r_t(s_t,z,s_{t+1})=(\phi(s_{t+1})-\phi(s_t))^\top z$. The 1-Lipschitz constraint $\|\phi(s_{t+1})-\phi(s_t)\|\leq d(s_{t+1},s_t)$ preserves the distance metric in the latent space. A skill conditioned policy $\pi(s,z)$ is trained via RL to maximize the WDM objective, thereby promoting trajectories that accumulate larger defined distance in the direction of $z$.

\subsection{Group Equivariance and Representation Theory}

If a transformation leaves a property of an object or system unchanged, that transformation is called a \emph{symmetry} \cite{bronstein2021geometric}. Symmetries satisfy the axioms of associativity, identity, inverse, and closure, and the collection of all such transformations forms a \emph{group}. In our setting, the relevant symmetry is planar rotation, represented by the continuous group $SO(2)$. For practical implementation, we approximate this continuous symmetry using a discrete cyclic subgroup $C_N$, where rotations are quantized in increments of $2\pi/N$.

A \emph{group representation} $\rho : G \to GL(n)$ maps each group element $g \in G$ to an invertible $n \times n$ matrix that describes how $G$ acts on a vector space. Examples include the \emph{trivial representation} $\rho_0(g)x = x$, which leaves the space unchanged, and the \emph{regular representation} on $\mathbb{R}^{|G|}$, which can be decomposed into \emph{irreducible representations} $\rho_k$ that cannot be further reduced. For $SO(2)$, the complex irreducible representations take the form $\rho_k(\theta) = e^{ik\theta}$. In practice, we use a real-valued form by separating real and imaginary parts into a two-dimensional feature vector. Any finite-dimensional representation $\rho$ of a compact group can be written, up to a change of basis, as a direct sum of irreducible representations:
\begin{equation*}
    \rho(g)
    = Q \left( \bigoplus_{k=1}^{K} \rho_k(g) \right) Q^{-1},
\end{equation*}
for some invertible matrix $Q$.

A function $f : X \to Y$ is \emph{equivariant} with respect to a group action if applying the group before or after $f$ is equivalent, i.e.,
\[
    \rho_Y(g) f(x) = f(\rho_X(g) x), \quad \forall g \in G, x \in X.
\]
If the output does not change under the group action on the input, $f(x) = f(\rho_X(g) x)$, then $f$ is \emph{invariant}. Throughout the paper we use the shorthand $gx$ to denote the group action $\rho(g)x$ for notational simplicity.

\subsection{Group-invariant MDPs}
\label{gmdp}
A group-invariant MDP \cite{wang2022mathrm, wang2022equivariant} builds on the notion of MDP homomorphism \cite{ravindran2001symmetries, ravindran2004approximate} and provides an abstract representation of an underlying MDP that respects symmetry. Denoted as $\mathcal{M}_G(\mathcal S,\mathcal A,P,R,G)$, this formulation encodes how a symmetry group $G$ acts on states and actions. When the reward and transition dynamics satisfy the group-invariance conditions
\begin{equation*}
    R(s,a)=R(gs,ga),\quad P(s'|s,a)=P(gs'|gs,ga),
\end{equation*}
The optimal policy and optimal value functions of the original MDP can be derived directly from the  symmetry-reduced MDP. This reduction provides both computational efficiency and structural guarantees tied to the underlying group.

\section{Problem Statement}
\label{sec:problem}
In line with previous unsupervised skill discovery approaches, the problem is formalized within the framework of a reward-free MDP, denoted as $\mathcal{M}=(\mathcal S,\mathcal A,P)$. $\mathcal S$, $\mathcal A$, and $P$ represent the state space, action space, and transition probabilities $P:\mathcal S\times \mathcal A\rightarrow \mathcal S$, respectively. The latent skill vector $z\in Z$ is sampled from the prior distribution, $z\sim p(z)$. The policy is conditioned by the skill vector, $\pi(a|s,z)$, the skill vector informing the policy to select the corresponding action sequence. During the training, the skill vector doesn't change during the episode, and the policy is trained without any reward supervision, only with the intrinsic reward.
In this framework, we leverage the group symmetric prior, prior knowledge of rotation symmetry, to enhance the sample efficiency and generalizability. We convert the original MDPs to group-invariant MDPs by assuming the group-invariance of the intrinsic reward function and transition probability. 

\begin{figure}[!t]
    \centering
    \includegraphics[width=\columnwidth]{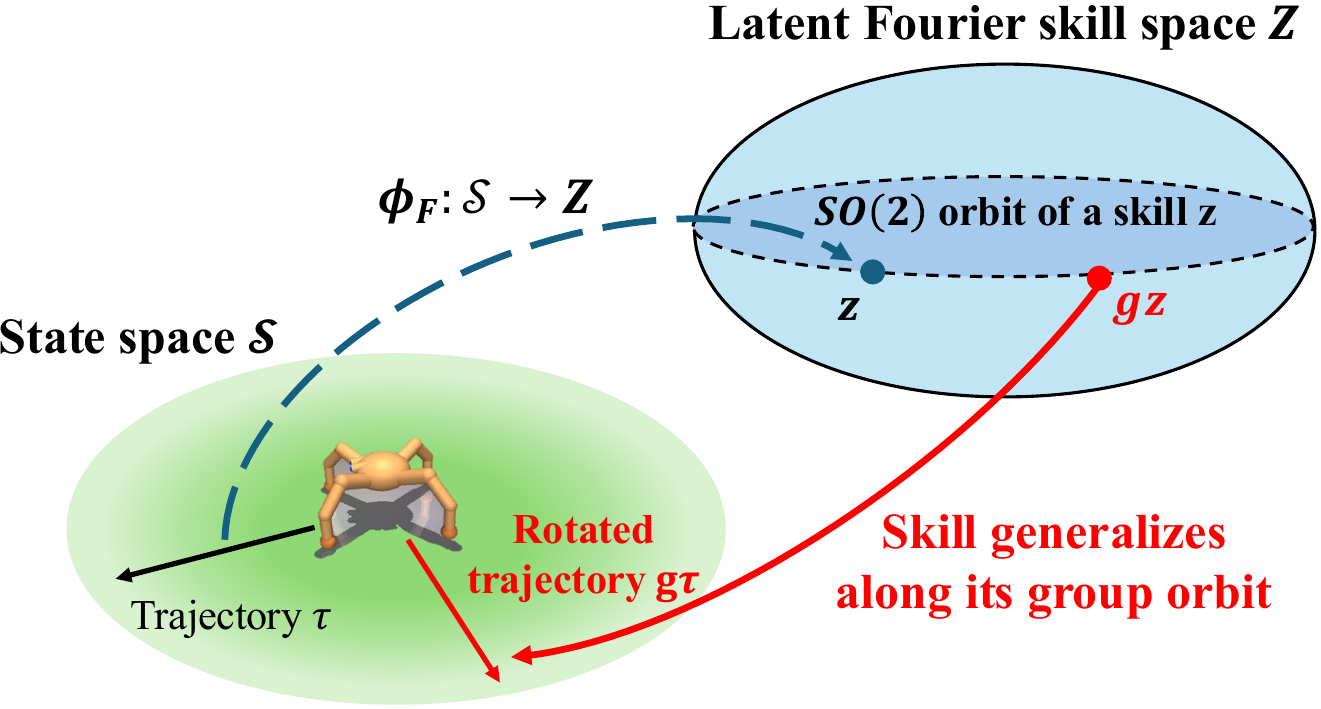}
    \caption{\textbf{Overview of Group-Invariant Skill Discovery (GISD).} Our method learns an equivariant mapping $\phi_F$ from the state space $\mathcal S$ to the \textbf{Group Fourier space} $Z$. By aligning state transitions with latent skill vectors $z$ in the group Fourier domain, the discovered skills inherently respect the underlying geometric symmetry. This structure enables generalization: a policy trained for a specific trajectory $\tau$ automatically generalizes to any group-transformed trajectory $g\tau$ by simply shifting the skill vector along its group orbit to $gz$.}
    \label{fig:overview}
\end{figure}

\section{Methods}
In this section, we present Group-Invariant Skill Discovery (GISD), a method that leverages the symmetry of the underlying dynamics to accelerate unsupervised skill discovery. We begin by showing that, in a group-invariant environment, the standard Wasserstein dependency measure (WDM) objective already admits symmetry-respecting optimal solutions. Motivated by this result, we then introduce a Group-Invariant WDM formulation and a practical Fourier-space parameterization and intrinsic reward.

\subsection{Theoretical Analysis: Symmetry of Optimal Solutions}
\label{sec:theory}
Before introducing GISD, we analyze the structure of optimal solutions to the standard WDM-based skill discovery problem. We show that, under group-invariant dynamics, restricting the search to symmetry-aware policies does not compromise optimality.

\begin{theorem}[Existence of Equivariant Optima]
    \label{thm:eq-opt}
    In an MDP where the dynamics, initial state distribution, skill prior, and ground metric are all group-invariant, the Wasserstein dependency measure admits a globally optimal solution $(\bar \pi, \bar f)$ such that
    \[
    \bar \pi(ga\mid gs, gz)=\bar \pi(a\mid s,z), \quad \bar f(gs,gz)=\bar f(s,z), \quad \forall g\in G.
\]
    In other words, among all WDM global maximizers, there exists at least one policy-function pair with an equivariant policy and a group-invariant $1$-Lipschitz function.
\end{theorem}
For proof, see Appendix~\ref{app:eq-opt}.
This theoretical guarantee implies that the exploration of non-equivariant policies—which correspond to redundant or suboptimal solutions in a symmetric environment—is unnecessary. This justifies restricting the search to the symmetry-aware function class without sacrificing global optimality.

\subsection{Group-Invariant Wasserstein Dependency Measure}
Theorem~\ref{thm:eq-opt} shows that in a group-invariant environment, the standard Wasserstein dependency measure (WDM) admits at least one global optimum where the policy is equivariant and the scoring function $f$ is group-invariant. This motivates explicitly restricting the scoring function class to symmetry-respecting functions: instead of hoping that symmetry emerges implicitly, we bake it into the objective.

Let $G$ act on the state space $\mathcal S$ and skill space $Z$, and let $d$ be a distance metric on $\mathcal S \times Z$ that is group-invariant under the joint action:
\[
d((gs,gz),(gs',gz')) = d((s,z),(s',z')),\quad \forall g\in G.
\]
We define the group-invariant scoring function class
$\mathcal{F}_{G}$ as the set of 1-Lipschitz functions $f:\mathcal S\times Z\to \mathbb R$ that are
invariant under the joint action of $G$:
\[
\mathcal{F}_{G} \\
:= \bigl\{ f
\;\big|\;
\|f\|_{L} \le 1,\; f(gs,gz) = f(s,z),\; \forall g\in G \bigr\}.
\]

\begin{definition}[Group-Invariant WDM]
    \label{def:gwdm}
    Given a group-invariant metric $d$ on $S\times Z$, the
    \emph{group-invariant} Wasserstein dependency measure is defined as
    \begin{equation*}
    I_{\mathcal W}^{G}(\mathcal S;Z)
    := \sup_{f \in \mathcal{F}_{G}}
      \mathbb{E}_{p(s,z)}[f(s,z)]
     - \mathbb{E}_{p(s)p(z)}[f(s,z)]).
    \end{equation*}
\end{definition}

This constrained objective guarantees invariance under the group action,
a property not generally enjoyed by the standard WDM.

\begin{proposition}
    \label{prop:gwdm}
    If the distance metric $d$ is group-invariant, then the
    group-invariant WDM satisfies
    \[
    I_{\mathcal W}^{G}(g\mathcal S;gZ) = I_{\mathcal W}^{G}(\mathcal S;Z)
    \]
    for all $g\in G$.
\end{proposition}

Proof is provided in Appendix~\ref{app:gwdm}.\\

\textbf{Reduction of the search space.}\quad
Compared to the standard WDM, which optimizes over all 1-Lipschitz functions, $I^G_{\mathcal W}$ restricts the scoring function class to the group-invariant subset $\mathcal F_G$. Based on
Theorem~\ref{thm:eq-opt}, this restriction does not eliminate any globally optimal symmetric solution, but acts as a geometric inductive bias: it rules out functions that encode spurious dependencies and focuses optimization on symmetry-consistent structure. In the following subsections, we show how to construct such group-invariant scoring functions via group averaging and how to parameterize them efficiently in Fourier space.

\subsection{Group-Averaged Scoring Functions}
The group-invariant WDM in Definition~\ref{def:gwdm} restricts the
admissible scoring functions (the 1-Lipschitz functions in the dual formulation)
to the invariant class $\mathcal{F}_{G}$. We now construct such
group-invariant scoring functions from arbitrary admissible $1$-Lipschitz
functions via a standard group averaging operation.

\begin{definition}[Group-averaged Scoring Function]
\label{def:group-avg}
For any $1$-Lipschitz function $f : \mathcal S \times Z \to \mathbb R$
(with respect to the metric $d$), we define its group average as
\begin{equation*}
    \tilde{f}(s,z) := \int_G f(gs,gz)\, d\mu(g),
\end{equation*}
where $\mu$ is the normalized Haar measure on $G$
(Appendix~\ref{app:haar}).
\end{definition}

This operator maps a general admissible scoring function $f$ to a group-invariant
function while preserving the $1$-Lipschitz constant:

\begin{proposition}[Properties of the group-averaged scoring function]
\label{prop:group-avg}
Let $f:\mathcal S\times Z\to \mathbb R$ be $1$-Lipschitz with respect to
a group-invariant metric $d$ on $\mathcal S\times Z$. Then its group
average $\tilde f$ is group-invariant and $1$-Lipschitz with respect to
$d$. In particular, $\tilde f \in \mathcal{F}_{G}$.
\end{proposition}

We provide the proof in Appendix~\ref{app:group-avg}.
The group averaging operator provides the precise link between the standard and group-invariant WDM: given any optimal scoring function $f^*$ in the unconstrained class, its group average $\tilde{f}^*$ is feasible in $\mathcal{F}_G$ (Proposition~\ref{prop:group-avg}) and, assuming an equivariant optimal policy, achieves the same objective value. Thus, for symmetric optima, the restricted and unconstrained WDMs coincide.



\subsection{Skill Discovery in the Group Fourier Space}
\label{sd_fourier}
In contrast to the decompositions of $f(s,z)$ used in prior works
\cite{park2022lipschitz, park2023controllability, park2023metra}, we
derive our parameterization from a group representation-theoretic
viewpoint. Assume a compact group $G$ acts on the state space
$\mathcal S$ and latent skill space $Z$. We consider the induced action
of the product group $G\times G$ on $\mathcal S\times Z$, given by
$(g_1, g_2)\cdot (s,z)=(g_1s, g_2z)$, and define
\begin{equation*}
    f_{g_1, g_2}(s,z):=f(g_1s, g_2z).
\end{equation*}
For each fixed $(s,z)$, the map $(g_1, g_2)\mapsto f_{g_1, g_2}(s,z)$
is square-integrable on $G\times G$, i.e., $f_{(\cdot, \cdot)}\in
L^2(G\times G)$. By the Peter--Weyl theorem
\cite{ceccherini2008harmonic}, this function admits an expansion in
terms of the irreducible representations of $G\times G$
(see Appendix~\ref{fourier} for details):
\begin{equation}
\label{f_decom}
f_{g_1,g_2}(s,z)
= \sum_{\rho,\sigma \in \hat G} d_\rho d_\sigma\,
    \mathrm{Tr}\!\big(A_{\rho,\sigma}(s,z)\,
    (\rho(g_1)\otimes \sigma(g_2))\big),
\end{equation}
where $\hat{G}$ denotes the set of irreducible representations of $G$, $d_\rho$ and
$d_\sigma$ are their dimensions, and $A_{\rho,\sigma}(s,z)$ are
coefficient matrices depending on $(s,z)$.

To enforce the group-invariance required in
Definition~\ref{def:gwdm}, we apply the group average over the diagonal
action $g\mapsto (g,g)$:
\[
\tilde f(s,z)
= \int_G f_{g,g}(s,z)\, d\mu(g)
= \int_G f(gs,gz)\, d\mu(g).
\]
Substituting Eq.~(\ref{f_decom}), we utilize Schur's Lemma
\cite{dummit_foote} and the orthogonality relations of irreducible representations:
\[
\int_G\rho(g)\otimes \sigma(g)\ d \mu(g)\neq 0 \Leftrightarrow \sigma\cong \rho^*.
\]
This averaging operation annihilates all non-matching components ($\rho\neq \sigma^*$) and projects the remaining terms onto the group-invariant scalar subspace. Consequently, the double sum collapses into a single sum over the spectrum:
\begin{equation*} 
\tilde f(s,z) = \sum_{\rho \in \hat G} \lambda_{\rho}(s,z), 
\end{equation*}
where each $\lambda_\rho (s,z)$ represents the scalar contraction of the matching representation blocks. 

We now parameterize these scalar coefficients in terms of
\emph{Fourier-space features} of $s$ and $z$. 
For each surviving irrep $\rho$, we introduce feature vectors
$\phi_{\rho}(s)$ and $\psi_{\rho}(z)$ in the corresponding representation
spaces and parameterize
\begin{equation*}
    \lambda_{\rho}(s,z)
    \approx \big\langle \phi_{\rho}(s),\, \psi_{\rho}(z)\big\rangle.
\end{equation*}
Stacking all irrep components into a single global feature vector, we obtain
the maps
\[
\phi_F : \mathcal{S} \to \mathbb{R}^d,
\quad
\psi_F : Z \to \mathbb{R}^d,
\]
and the group-averaged scoring function takes the inner-product form:
\begin{equation*} 
\tilde{f}(s,z) \approx \big\langle \phi_F(s), \psi_F(z) \big\rangle. 
\end{equation*}
By construction, $\phi_F$ maps states into the spectral domain and satisfies the \emph{equivariance} property:
\begin{equation}
\label{equiv_prop} 
\phi_F(gs) = \rho_{F}(g)\phi_F(s), 
\end{equation}
where $\rho_{F}(g)$ denotes the block-diagonal direct sum of unitary irreducible representations acting on the feature vector. Following METRA \cite{park2023metra}, we
set $\psi_F(z)=z$, treating the latent skill $z$ as a
unit-norm vector in   Fourier space. This choice, together with a 1-Lipschitz parameterization of $\phi_F$,
is compatible with the $1$-Lipschitz constraint required by the WDM
objective.\\

\textbf{Frequency Filtering and Interpretability.}
A key advantage of working in Fourier space is the explicit separation
of geometric frequencies. The feature vector $\phi_F(s)$ is a
concatenation of coefficients corresponding to different irreducible representations (frequencies). Unlike unstructured latent spaces, this allows us to manually emphasize or suppress specific symmetries by applying a frequency mask $M$:
\begin{equation*}
    \tilde f(s,z)=\big\langle M \cdot \phi_F(s),\, z\big\rangle.
\end{equation*}
For example, by selecting only the trivial representation ($\rho_0$), we
can discover purely group-invariant skills. By selecting higher-order frequencies, we encourage skills that are sensitive to the group action. This provides a degree of interpretability and controllability unavailable in standard unsupervised skill discovery.

\subsection{Group-Invariant MDP for Skill Discovery}
\label{sec:induced_mdp}

Using the group-invariant scoring function parameterization from
Section~\ref{sd_fourier}, we now derive the intrinsic reward used for
skill discovery and show that it is itself group-invariant.

Recall that METRA \cite{park2023metra} optimizes a WDM-style objective
over trajectories (Eq.~\eqref{metra wdm}).
In our framework, we restrict to group-invariant scoring functions and
parameterize the group-averaged scoring function as an inner product in Fourier
space,
\[
\tilde f(s,z) \approx \big\langle \phi_F(s), z \big\rangle,
\]
where $\phi_F$ is equivariant with respect to the group action
(Section~\ref{sd_fourier}). Substituting this parameterization into
Eq.~\eqref{metra wdm} leads to
\begin{equation}
\begin{aligned}
\label{eq:objective-fourier}
I_{\mathcal W}(\mathcal S_T; Z)
&\approx
\sup_{\|\phi_F\|_{L}\leq 1}
\sum_{t=0}^{T-1}
\Big(
    \mathbb{E}_{\tau,z}
    \big[
        \langle \phi_F(s_{t+1}) - \phi_F(s_t), z \rangle
    \big] \\
&\quad
    - \mathbb{E}_{\tau}
    \big[
        \langle \phi_F(s_{t+1}) - \phi_F(s_t), \mathbb{E}_{z}[z] \rangle
    \big]
\Big).
\end{aligned}
\end{equation}
As in METRA, we choose $p(z)$ to have zero mean $\mathbb{E}_{p(z)}[z] = 0$, in which case the second term vanishes and
\eqref{eq:objective-fourier} reduces to
\begin{equation}
\label{eq:objective-simplified}
I_{\mathcal W}(\mathcal S_T; Z)
\approx
\sup_{\|\phi_F\|_{L}\leq 1}
\mathbb{E}_{\tau,z}
\Bigg[
    \sum_{t=0}^{T-1}
    \langle \phi_F(s_{t+1}) - \phi_F(s_t), z \rangle
\Bigg].
\end{equation}

Because the sum in \eqref{eq:objective-simplified} is telescoping, it is
natural to interpret the incremental term as an intrinsic reward:
\begin{equation}
\label{eq:reward-def}
    r(s_t, z, s_{t+1})
    := \big\langle \phi_F(s_{t+1}) - \phi_F(s_t),\, z \big\rangle.
\end{equation}
We now show that this reward is group-invariant. Let $h \in G$ act on
states and skills as $s' = hs$ and $z' = hz$. Using the equivariance of
$\phi_F$ (Eq.~\eqref{equiv_prop}) and the fact that the representation
$\rho(h)$ is unitary, we obtain
\begin{align*}
r(hs_t, hz, hs_{t+1})
&= \big\langle
        \phi_F(hs_{t+1}) - \phi_F(hs_t),\, hz
   \big\rangle \\
&= \big\langle
        \rho(h)\big(\phi_F(s_{t+1}) - \phi_F(s_t)\big),
        \rho(h) z
   \big\rangle \\
&= \big\langle
        \phi_F(s_{t+1}) - \phi_F(s_t),\, z
   \big\rangle \\
&= r(s_t, z, s_{t+1}),
\end{align*}
where the third equality uses unitarity of $\rho(h)$:
$\langle \rho(h)x, \rho(h)y \rangle = \langle x, y \rangle$.
Thus, the intrinsic reward \eqref{eq:reward-def} is invariant under the joint group action on states and skills. 

This invariance leads to a significant theoretical result for our learning process. Because the environment dynamics are group-invariant (by assumption) and the derived intrinsic reward $r$ is group-invariant (as shown above), the skill discovery problem constitutes a group-invariant MDP \cite{wang2022equivariant}. This structural property guarantees the existence of an optimal equivariant policy $\pi^*$.\\

We now connect this architectural choice back to the theoretical guarantees of our objective:

\begin{remark}[Exactness for Equivariant Policies]
    While Theorem~\ref{thm:eq-opt} guarantees the existence of a symmetric global optimum, our method parameterizes the policy $\pi$ to be equivariant based on the group-invariant MDP. 
    In a group-invariant MDP, an equivariant policy induces a group-invariant joint distribution $p(s,z)$ (Lemma~\ref{lem:occ-invariance}). 
    Under such a symmetric distribution, the optimal scoring function $f^*$ for the unconstrained WDM is inherently group-invariant (as non-symmetric components of the scoring function average to zero or contribute nothing to the objective).
    Therefore, for our equivariant policy class, optimizing the restricted objective $I^G_{\mathcal W}$ is mathematically equivalent to optimizing the standard WDM, ensuring no loss of expressivity.
\end{remark}

\begin{algorithm}[t]
\caption{Group-Invariant Unsupervised Skill Discovery}
\label{algo}
\begin{algorithmic}[1]
    \STATE \textbf{Initialize:} Equivariant policy $\pi(a|s,z)$, Fourier mapping $\phi_F(s)$, dual variable $\lambda$, replay buffer $\mathcal{D}$
    \FOR{$\text{epoch} = 1, \dots, N$}
        \FOR{$\text{episode} = 1, \dots, M$}
            \STATE Sample skill $z \sim p(z)$
            \STATE Collect trajectory $\tau$ using $\pi(a|s,z)$
            \STATE Store transitions $(s, a, s', z)$ in $\mathcal{D}$
        \ENDFOR
        
        \STATE \textbf{1. Update Discriminator} $\phi_F$:
        \STATE Optimize $\mathcal J_\phi$ (Eq.~(\ref{phi_F_loss}))
        
        \STATE \textbf{2. Update Dual Variable} $\lambda$:
        \STATE Optimize $\mathcal J_\lambda$ (Eq.~(\ref{lam_loss}))
        
        \STATE \textbf{3. Update Policy} $\pi$:
        \STATE Compute rewards $r = \langle \phi_F(s')-\phi_F(s), z \rangle$
        \STATE Update $\pi$ using symmetry-aware SAC \cite{wang2022equivariant, chang2025partial}
    \ENDFOR
\end{algorithmic}
\end{algorithm}

\subsection{Group Symmetry in Downstream Tasks}
After skill discovery, prior works \cite{park2022lipschitz, park2023controllability, park2023metra} reuse the discovered skills for downstream tasks via a hierarchical architecture. A high-level policy $\pi^h(z\mid s,\text{goal})$ selects a latent skill conditioned on the current state and goal, while a low-level policy $\pi^l (a\mid s,z)$ executes primitive actions using the pretrained skill-conditioned policy (with parameters frozen).

This interaction induces a \textbf{Fixed-interval Semi-MDP} \cite{sutton1999between}, where the high-level transition dynamics depend on the cumulative effect of the low-level policy over $k$ steps. We formally define this transition probability $P_k$: 

\begin{definition}[Fixed-interval High-level Transition]
Given a fixed interval $k$, the high-level transition kernel induced by the low-level policy $\pi^l$ is defined as the $k$-step roll-out:
\begin{equation}
\label{eq:Pk-def}
\begin{aligned}
    &P_k(s_{t+k} \mid s_t, z) := \int \dots \int \prod_{i=0}^{k-1} \bigg[ \\
    &\quad \pi^l(a_{t+i} \mid s_{t+i}, z) \cdot
    P^l(s_{t+i+1} \mid s_{t+i}, a_{t+i})
    \bigg] \, d\tau_{int}
\end{aligned}
\end{equation}
where $P^l$ denotes the one-step transition kernel of the underlying MDP, and the integration is over all intermediate states and actions $\tau_{int} = \{a_{t:t+k-1}, s_{t+1:t+k-1}\}$.
\end{definition}

For downstream tasks, the extrinsic reward can typically be chosen to be group-invariant (e.g., sparse goal-reached reward). We now show that if the environment dynamics and low-level skill policy are symmetric, then the induced high-level semi-MDP is itself group-invariant.

\begin{theorem}[Group-invariant Fixed-interval Semi-MDP] 
\label{theorem:semiMDP} 
Assume the environment dynamics are group-invariant and the pretrained skill policy is equivariant: \begin{align*} 
P^l(gs' \mid gs, ga) &= P^l(s' \mid s, a) \\ 
\pi^l(ga \mid gs, gz) &= \pi^l(a \mid s, z) \quad \forall g \in G. 
\end{align*} 
Then, the induced high-level transition kernel (Eq.~(\ref{eq:Pk-def})) satisfies group invariance: \begin{equation*} 
P_k(g s_{t+k} \mid g s_t, g z) = P_k(s_{t+k} \mid s_t, z), \qquad \forall g \in G. 
\end{equation*} 
Consequently, the fixed-interval semi-MDP is group-invariant whenever the high-level reward is group-invariant. \end{theorem} 
The proof is provided in Appendix~\ref{gsmdp}.\\

\textbf{Generalizable Skills.}\quad Theorem~\ref{theorem:semiMDP} establishes that our pretrained skills induce a group-invariant fixed-interval semi-MDP. Since the skills are represented in the Fourier space and encode symmetric features by construction, an equivariant high-level policy $\pi^h(z \mid s, g)$ can systematically exploit this structure.
Because the transition dynamics are group-invariant, a high-level policy trained in one group element can generalize to any other group elements by simply transforming the skill vector $z$ according to the group action (Fig.~\ref{fig:overview}). This allows for highly sample efficient downstream task training.

\subsection{Practical Implementation}
To finalize the training objective, we must specify the distance metric $d$ for the Lipschitz constraint. 
We adopt the temporal distance $d_T(s,s')$ \cite{park2023metra}, which measures the expected number of steps between states. Since the underlying MDP dynamics are group-invariant, $d_T$ is inherently group-invariant (see Appendix~\ref{gtd} for proof), satisfying the metric requirement of Definition~\ref{def:gwdm}.

Our training procedure follows the dual gradient descent method of METRA, but substitutes the standard discriminator with our equivariant Fourier mapping $\phi_F$. The discriminator parameters are updated to minimize the following loss, which balances skill alignment against the Lipschitz constraint via a Lagrange multiplier $\lambda$:
\begin{align}
    \label{phi_F_loss}
    \mathcal{J}_{\phi} &= \mathbb{E}_{\mathcal{D}} \Big[
        \langle \phi_F(s')-\phi_F(s), z \rangle \notag \\
        &\quad + \lambda \cdot \min\bigl(\epsilon, 1-\|\phi_F(s')-\phi_F(s)\|_2^2\bigr)
    \Big], \\
    \label{lam_loss}
    \mathcal{J}_{\lambda} &= \mathbb{E}_{\mathcal{D}} \Big[
        \lambda \cdot \min\bigl(\epsilon, 1 - \|\phi_F(s')-\phi_F(s)\|_2^2 \bigr)
    \Big],
\end{align}
where $\mathcal D$ is the replay buffer.
The policy $\pi(a\mid s,z)$ is updated using group equivariant SAC \cite{wang2022equivariant} or PE-SAC \cite{chang2025partial} to maximize the intrinsic reward $r=\langle\phi_F(s')-\phi_F(s), z\rangle$.

\section{Experiments}

We evaluate GISD against the state-of-the-art unsupervised skill discovery method, METRA \cite{park2023metra}.
Our goal is to understand how explicitly exploiting symmetry during skill discovery and downstream task training affects (i) sample efficiency, (ii) state-space coverage, and (iii) downstream task performance.
We follow the overall experimental setup of METRA and consider both state-based and pixel-based locomotion with continuous latent skills.

\subsection{Benchmark environments and Symmetry Structure}
We use two benchmark locomotion tasks with known geometric symmetries.

\paragraph{State-based Ant ($C_4$ symmetry).}
For state-based skill discovery, we use the MuJoCo Ant environment~\cite{brockman2016openai} with a planar $(90^\circ)$ rotational symmetry, modeled as the cyclic group $C_4$.
To faithfully preserve this symmetry, we \emph{do not} apply observation normalization (e.g., running mean/variance) to the state input. Such preprocessing typically warps coordinates and breaks equivariance.
In this setting, we restrict the Fourier representation to the frequency-$1$ $C_4$ irrep and discover 2D latent skills in this subspace, matching the skill dimensionality used in METRA (as introduced in Section~\ref{sd_fourier}). Intuitively, this emphasizes directional, rotation equivariant behaviors (e.g., "move in a given heading"), which are most relevant in Ant locomotion.
We empirically observed that incorporating invariant components had little benefit and could introduce noise in discovery, so we focus on the equivariant features in this work. 

\paragraph{Pixel-based Quadruped ($D_1 \cong C_2$ / flip symmetry}).
For pixel-based skill discovery, we adopt the vision-based quadruped environment from the DeepMind Control Suite (DMC)~\cite{tunyasuvunakool2020}), following the METRA \cite{park2023metra} setup.
METRA encodes position information by making the floor color vary smoothly as a function of the agent's $(x,y)$ coordinates.
To fully expose the horizontal flip symmetry to the agent, we instead redesign the floor so that its color varies \emph{radially} from the origin.
This maintains a consistent appearance under left-right flips while still revealing global position, and matches the $C_2$ (horizontal flip) symmetry.

The group $C_2$ has a 2D real representation consisting of an invariant channel and a sign-flip channel. To keep the comparison with METRA fair, we use a 4D latent skill space by concatenating two such Fourier-feature vectors.
This can be viewed as learning two independent Fourier components. In practice, this provided a good balance between expressivity and parity with the baseline.

\paragraph{Symmetry-aware Policy.}
On the control side, we use Partially Equivariant SAC \cite{chang2025partial} for the state-based Ant, which is robust to local symmetry-breaking effects (e.g., contact asymmetries).
For the pixel-based quadruped, we use Equivariant SAC \cite{wang2022equivariant}.
In both environments, a high-level policy $\pi^h (z \mid s,g)$ selects a skill vector $z$ every $K$ steps to pursue a goal $g$, while a low-level policy $\pi^l (a\mid s,z)$ executes the learned skill.
The skill policy is pretrained and then forzen for downstream training, as in METRA.

\begin{figure}[!t]
    \centering
    \includegraphics[width=0.9\columnwidth]{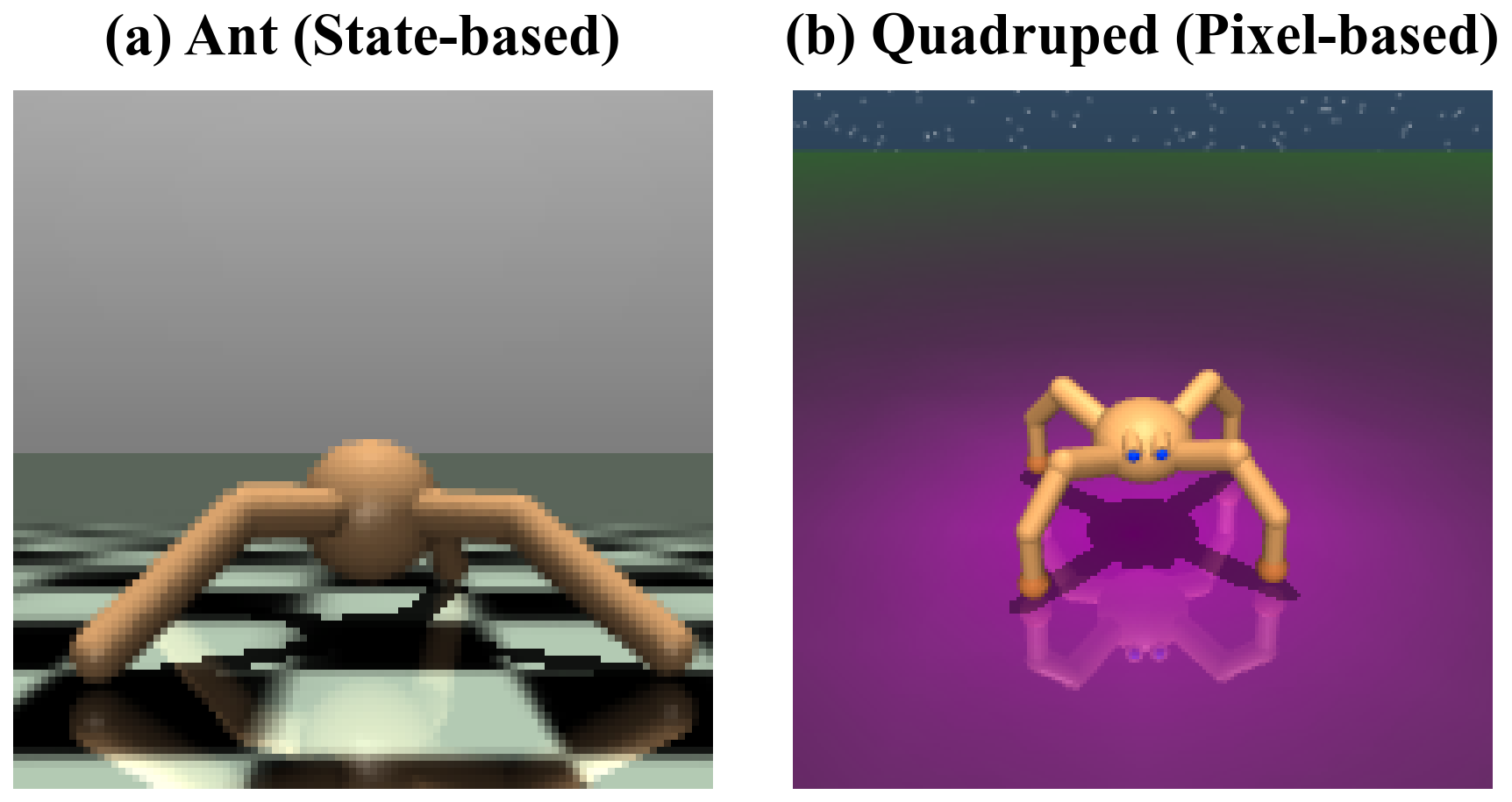}
    \caption{\textbf{Benchmark Environments.}
    We evaluate GISD in state-based Ant and pixel-based Quadruped locomotion.
    In the pixel-based setting, we redesign the floor so that color varies radially from the origin,
    making the observations consistent with the horizontal flip symmetry.}
    \label{fig:env}
\end{figure}

\subsection{Evalutation Metric}
We evaluate both skill discovery and downstream task performance.

\paragraph{Staet-space coverage.}
Following METRA \cite{park2023metra}, we measure how well skills explore the environment by computing the coverage of the agent's $(x,y)$ positions.
At evaluation, we sample 48 latent skill vectors uniformly from the skill prior, roll each for a fixed horizon, and aggregate the resulting trajectories.
We then quantify coverage using the fraction of grid cells visited in the task-relevant region of the plane.
For the state-based Ant, we restrict the coverage region to $[-30, 30]^2$, which corresponds to the maximum distance the agent can reach under the goal-sampling scheme in the downstream tasks. This avoids counting irrelevant areas.

\paragraph{Downstream task.}
We next assess how useful the discovered skills are for solving goal-conditioned tasks.
We train a high-level policy ($\pi^h (z\mid s,g)$ on top of the frozen skill policy $(\pi(a\mid s,z)$.

In the state-based Ant environment, each episode is a multi-goal task:
Whenever the current goal is reached-or after $K$ steps-a new goal is sampled uniformly around the current position (we follow METRA and use a sampling range of $[-7.5, 7.5]^2$.

We report the average return as a function of training iterations during downstream task training, averaged over multiple random seeds.

\begin{figure}[!tb]
    \centering
    \includegraphics[width=0.9\columnwidth]{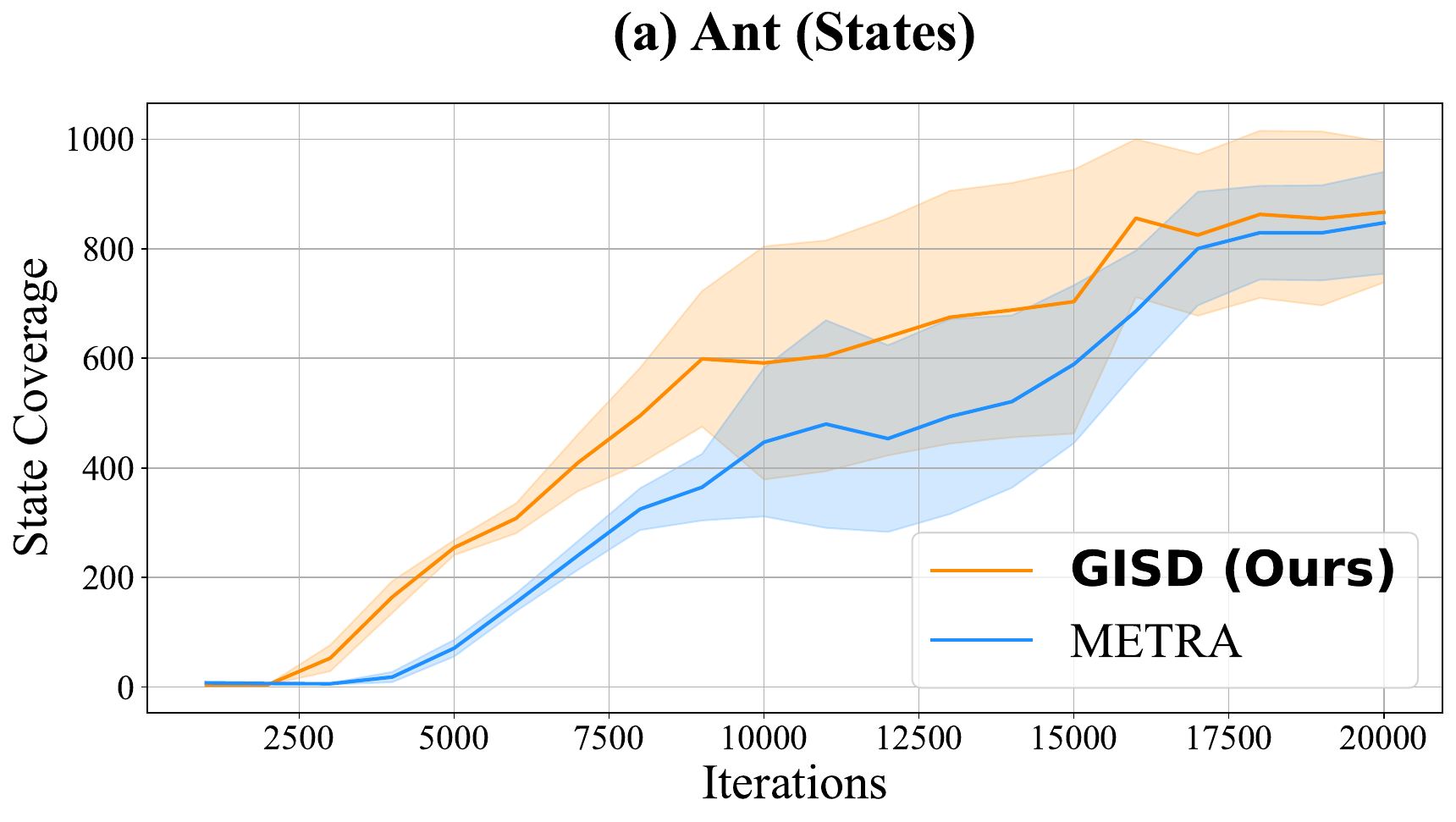}
    \vspace{0.8em} 
    \includegraphics[width=0.9\columnwidth]{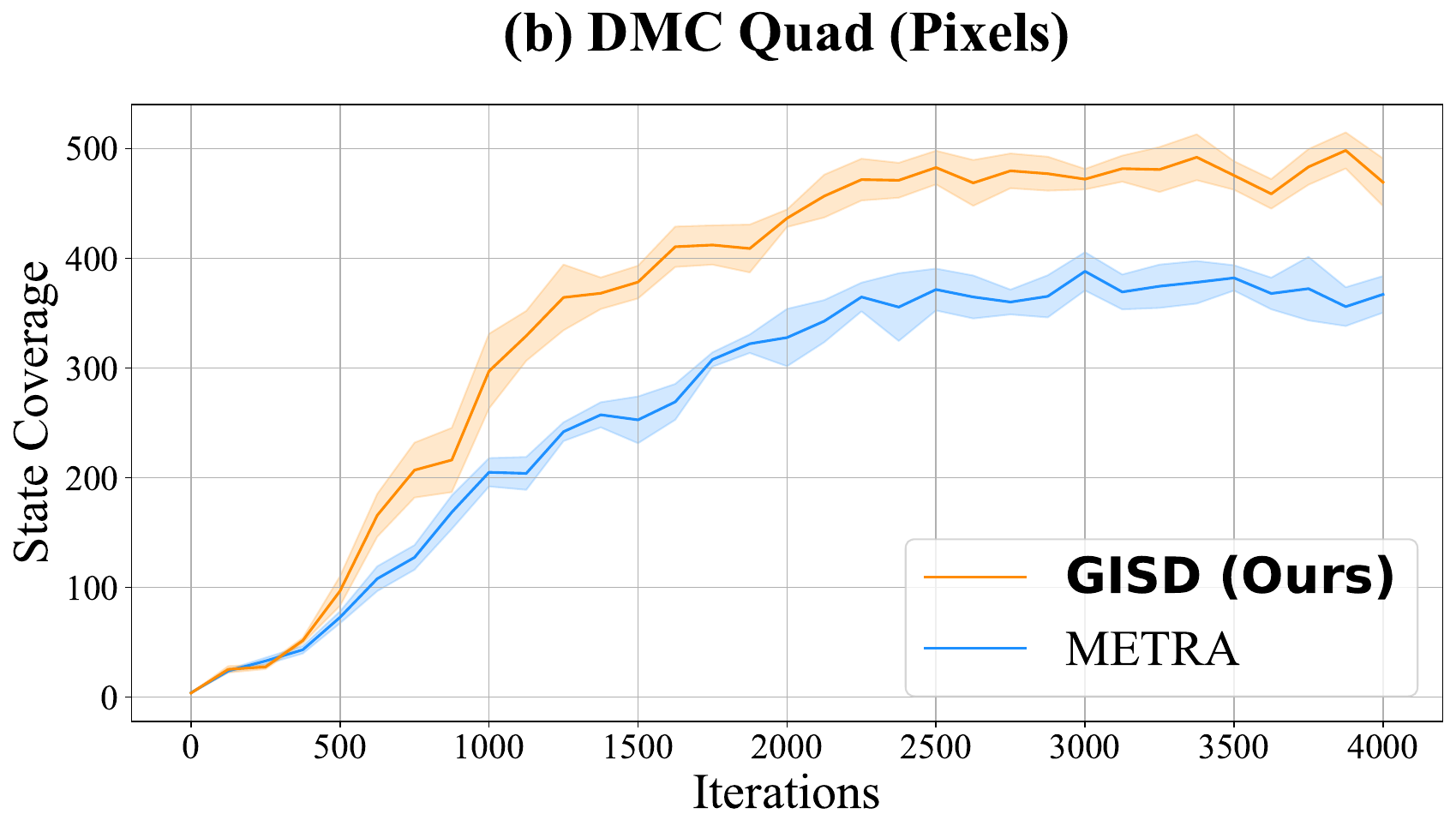}
    \caption{\textbf{State-space coverage during skill discovery.}
    \textbf{(a)} State-based Ant (4 seeds). 
    \textbf{(b)} Pixel-based Quadruped (5 seeds).
    Shaded regions show standard error. 
    GISD achieves higher coverage and better sample efficiency than METRA in both environments.}
    \label{fig:skill_discovery}
\end{figure}

\subsection{Results and Analysis}

\paragraph{Skill discovery and state coverage.}
Figure~\ref{fig:skill_discovery} shows the evolution of $(x,y)$-coverage during skill pretraining.
Across both state-based Ant and pixel-based quadruped, GISD achieves broader state-space coverage with higher sample efficiency than METRA. 
The shaded regions indicate standard error across seeds.
In the Ant environment, the variance is larger due to the lack of observation normalization, which makes the learning dynamics more sensitive to initialization.

To qualitatively understand the learned behaviors, Figure~\ref{fig:skill_discovery_vis} visualizes trajectories induced by 48 randomly sampled skills.
For GISD, skills tile the plane in a symmetry-consistent way:
In Ant, each direction has rotationally related counterparts (approximately obeying the $C_4$ symmetry).
In Quadruped, up-down reflected skills exhibit mirrored trajectories under the flip symmetry.
In contrast, METRA often learns clusters of redundant skills that differ only slightly in direction or fail to cover all symmetric modes, which matches the quantitative coverage gap.

\begin{figure}[!tb]
    \centering
    \includegraphics[width=\columnwidth]{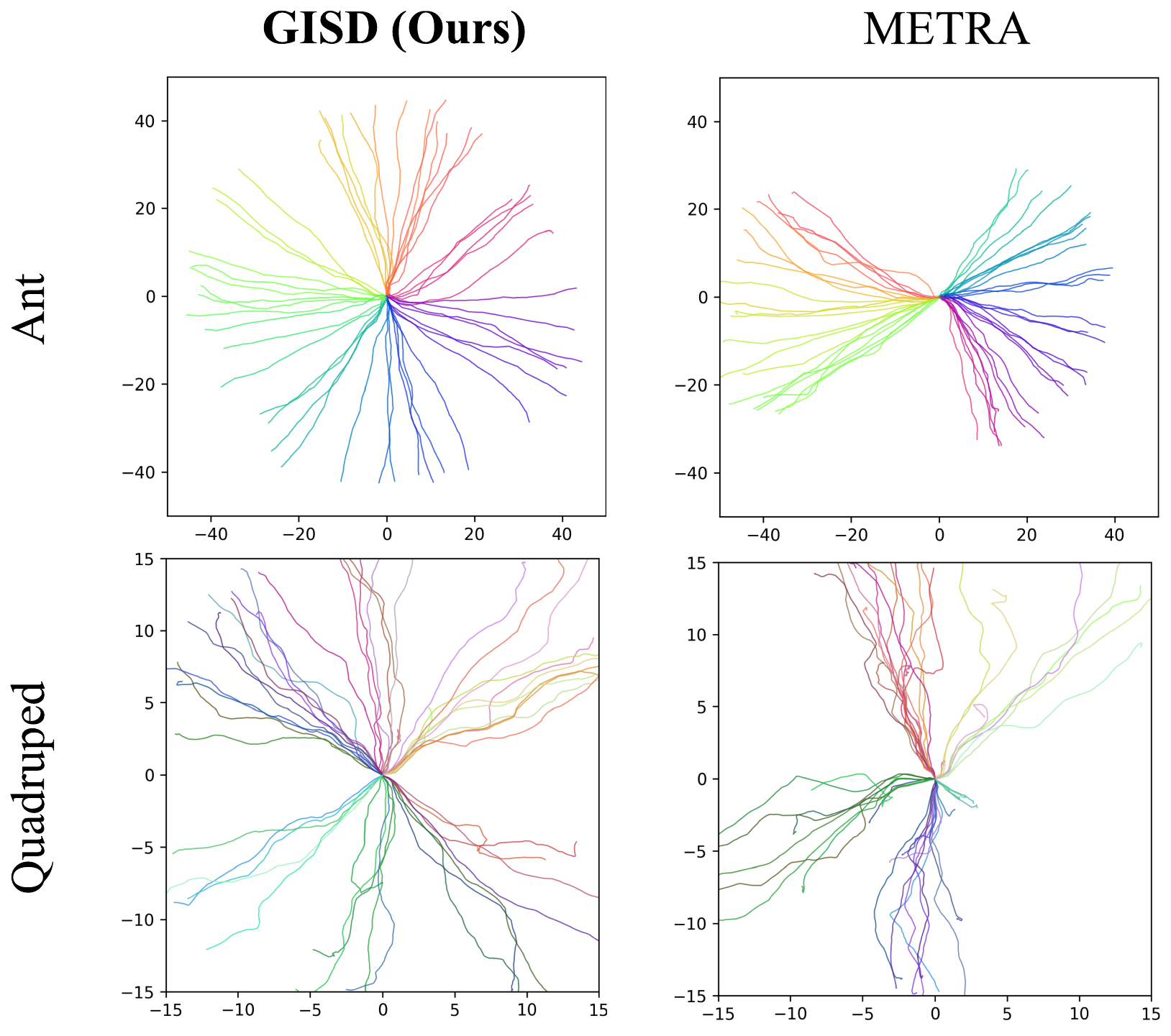}
    \caption{\textbf{Visualization of discovered skills.}
    We plot trajectories for 48 randomly sampled skills. Colors indicate different latent skills.
    GISD discovers symmetry-consistent skills (Ant: \(C_4\) rotations; Quadruped: horizontal flips), 
    leading to substantially broader and more uniform state-space coverage than METRA.}
    \label{fig:skill_discovery_vis}
\end{figure}

\paragraph{Downstream performance.}
Figure~\ref{fig:downstream} compares downstream goal-reaching performance.
In both the state-based Ant and pixel-based DMC Quadruped tasks, GISD attains higher returns than METRA and reaches strong performance with fewer environment steps.
Figure~\ref{fig:downstream_vis} visualizes the Fourier skill vectors selected by the high-level policy $\pi^h(z \mid s, g)$ in the DMC Quadruped environment. Starting from a shared initial state with horizontally mirrored goals, the policy chooses skills whose invariant components are nearly identical, while the equivariant components exhibit approximate sign flips. This behavior reflects the underlying flip symmetry and indicates that the Fourier-space skills learned by GISD are both diverse and systematically reusable across transformed tasks.

\begin{figure}[!tb]
    \centering
    \includegraphics[width=0.9\columnwidth]{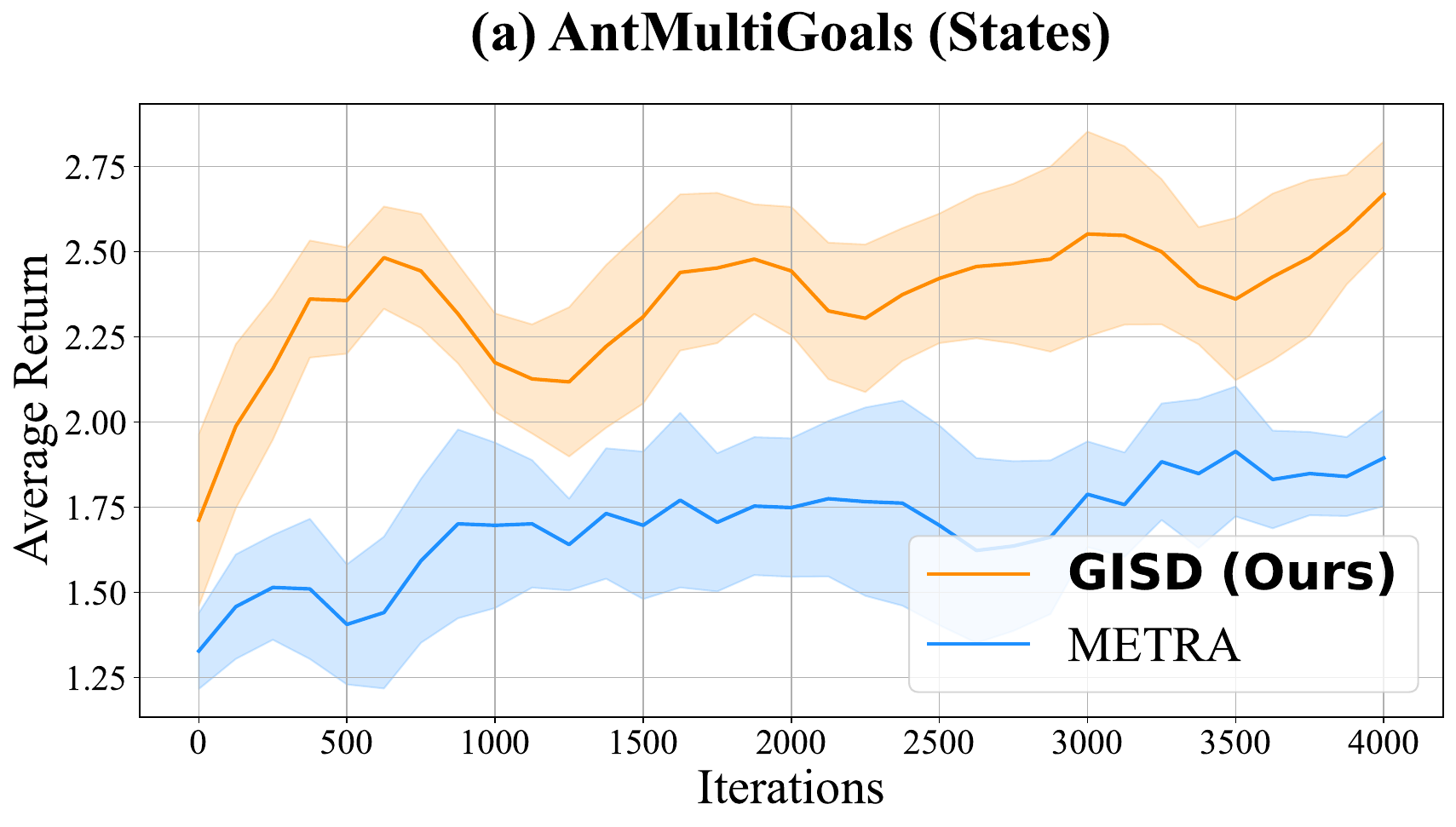}
    \vspace{0.8em} 
    \includegraphics[width=0.9\columnwidth]{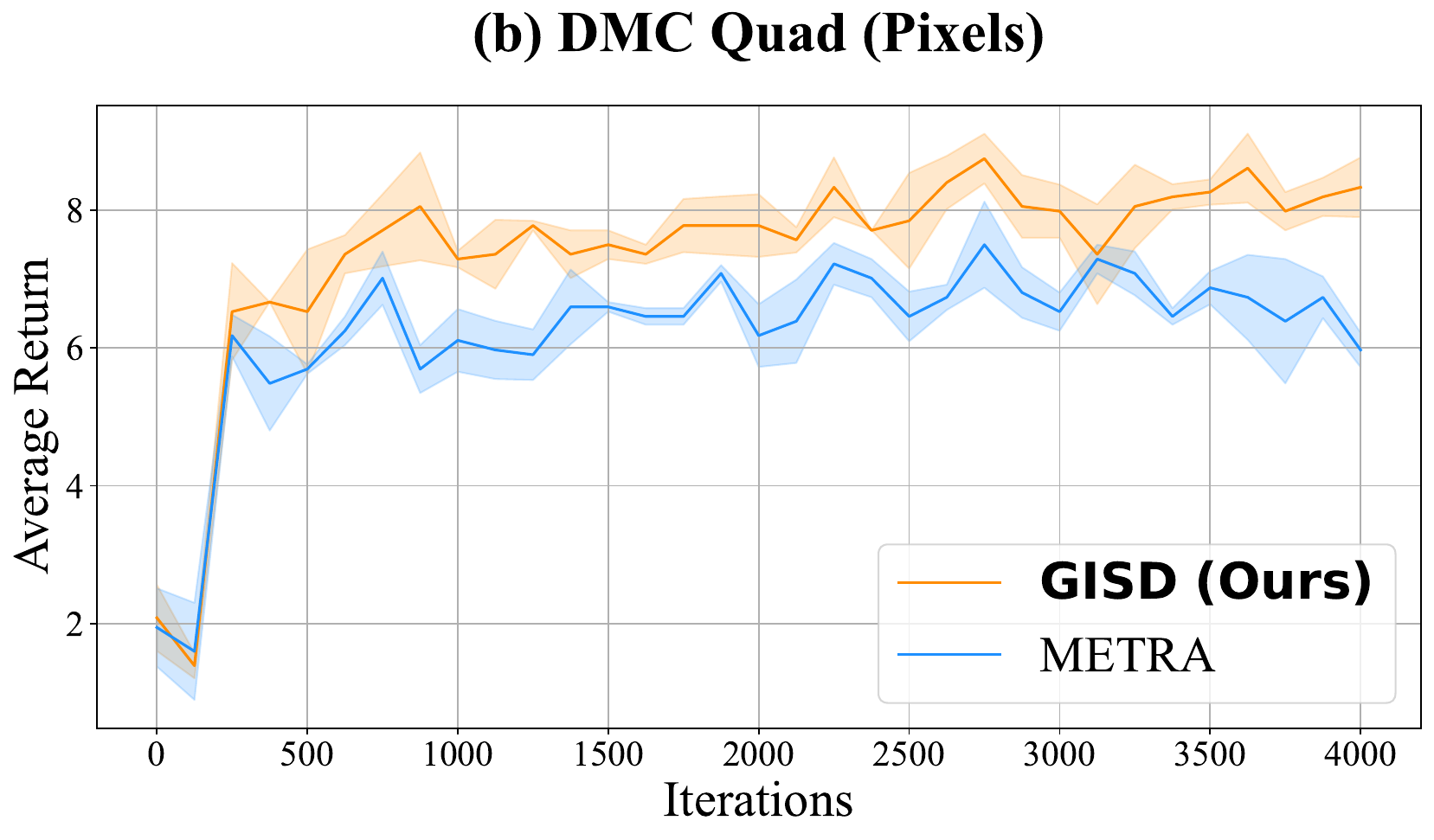}
    \caption{\textbf{Downstream task performance.}
    \textbf{(a)} Average return in the state-based Ant environment (4 seeds).
    \textbf{(b)} Average return in the pixel-based Quadruped environment (3 seeds).
    Shaded regions show standard error. 
    GISD consistently achieves higher performance and better sample efficiency than METRA.}
    \label{fig:downstream}
\end{figure}

Overall, these results support our main claim: 
By constraining the scoring function to a group-invariant function class and parameterizing it in Fourier space, GISD discovers symmetry-aligned skills that explore more effectively and transfer more efficiently to downstream control.

\begin{figure}[!tb]
    \centering
    \includegraphics[width=\columnwidth]{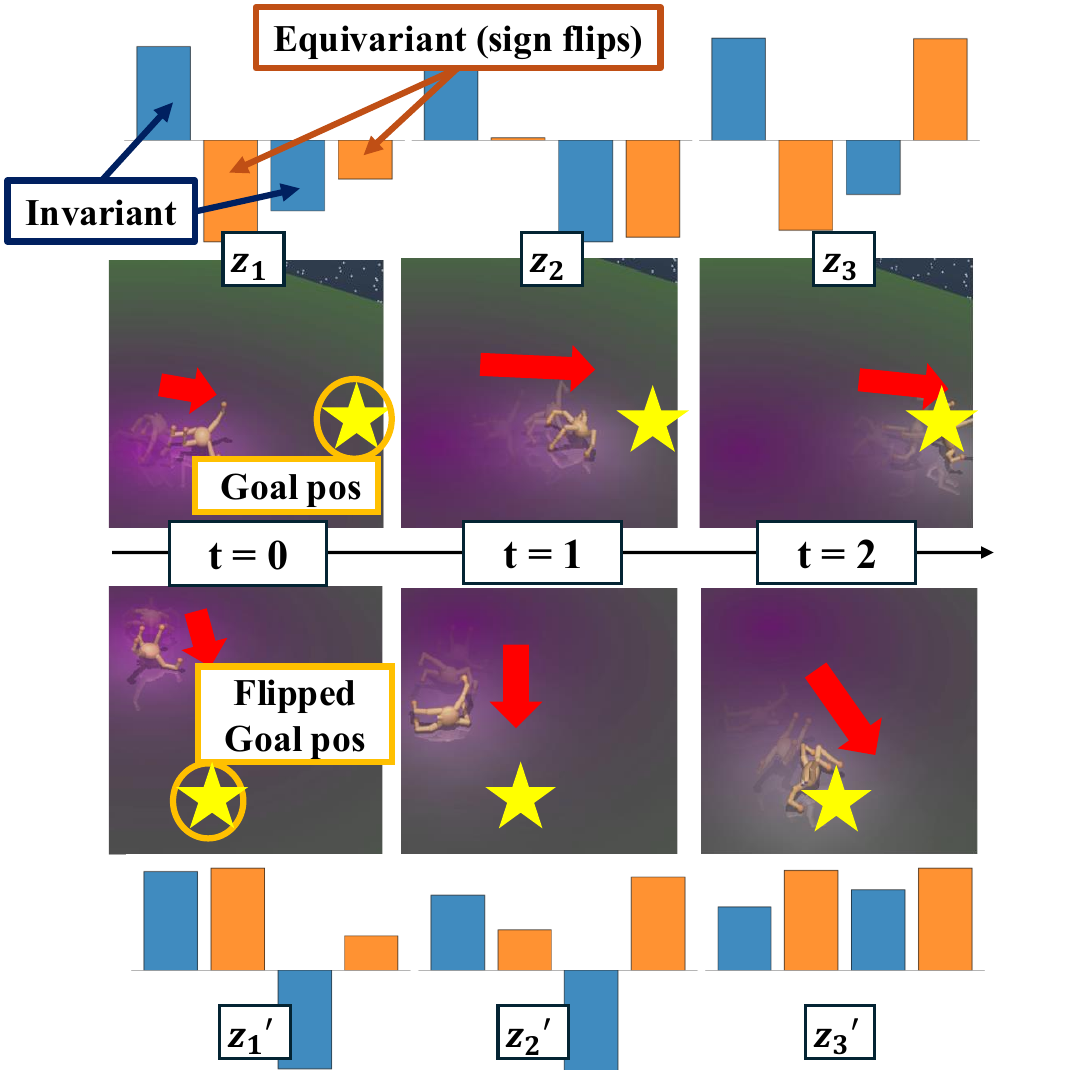}
    \caption{\textbf{Visualization of downstream task in DMC Quadruped.} Starting from identical initial states with horizontally flipped goal positions, we visualize the Fourier skill vectors predicted by the high-level policy over time. Blue bars denote the invariant components, while orange bars denote the equivariant components, which should flip sign under the flip transformation. At early time steps ($t=0,1$), the invariant features closely match, and the equivariant features appear with approximately opposite signs and similar magnitudes. By $t=2$, the components begin to diverge, reflecting accumulated symmetry-breaking during execution (e.g., the quadruped's rolling behavior under each skill).}
    \label{fig:downstream_vis}
\end{figure}

\section{Conclusions and Limitations}
In this work, we introduced \textbf{Group-Invariant Skill Discovery}, a framework that exploits group equivariance in skill discovery to improve sample efficiency and generalization. By enforcing group invariance in the Wasserstein dependency measure, we recast the objective as a symmetry-aware MDP and obtain latent skills that are consistent with the underlying group structure, which can then be effectively reused in downstream tasks.
Our experiments validate the approach on both state-based and vision-based locomotion benchmarks, but a natural next step is to test it on more challenging domains such as Humanoid control. The method also inherits common limitations of equivariant networks, including increased computational cost and the need to specify the relevant symmetry group a priori.
Looking forward, we plan to extend this framework to manipulation tasks, where skill discovery is notoriously difficult but rich symmetries are often present. Another promising direction is to study how the symmetric structure of the discovered skills can be more directly exploited in downstream planning and policy learning.


\newcommand{\wdm}{\mathrm{WDM}}
\newcommand{\E}{\mathbb{E}}

\section*{Appendix A.  Theoretical Details}

\subsection{Proof of Theorem \ref{thm:eq-opt}}
\label{app:eq-opt}

\theoremRestate{thm:eq-opt}
\begin{theoremR}[Existence of Equivariant Optima]
In an MDP where the dynamics, initial state distribution, skill prior, and ground metric are all group-invariant, the Wasserstein dependency measure admits a globally optimal solution $(\bar \pi, \bar f)$ such that
    \[
    \bar \pi(ga\mid gs, gz)=\bar \pi(a\mid s,z), \quad \bar f(gs,gz)=\bar f(s,z), \quad \forall g\in G.
\]
In other words, among all WDM global maximizers, there exists at least one policy-function pair with an equivariant policy and a group-invariant $1$-Lipschitz function.
\end{theoremR}

\begin{proof}
We define the Kantorovich dual objective $J(\pi, f)$ for a policy $\pi$ and a $1$-Lipschitz function $f: \mathcal{S} \times \mathcal{Z} \to \mathbb{R}$ as:
\begin{align*} 
J(\pi, f) := \E_{\substack{z \sim p(z) \ s \sim p_{\pi}(s \mid z)}}[ f(s,z) ] - \E_{\substack{z \sim p(z) \ s \sim p_{\pi}(s)}} [ f(s,z) ]. 
\end{align*}
where $p_{\pi}(s \mid z)$ is the conditional state occupancy measure induced by $\pi$, and $p_{\pi}(s)$ is the corresponding marginal. The Wasserstein Dependency Measure (WDM) is defined as
\[
\wdm(\pi) = \sup_{\lVert f \rVert_L \le 1} J(\pi, f).
\]
Let $(\pi^*, f^*)$ be a global maximizer of the problem $\max_{\pi} \wdm(\pi)$.
By the assumptions of Theorem \ref{thm:eq-opt}, the tuple $(\mathcal{S}, \mathcal{A}, \mathcal{Z})$ is equipped with a measure-preserving group action, the environment dynamics are equivariant, and the ground metric is invariant.

Consequently, the skill discovery optimization problem can be formulated as finding the pair $(\pi, f)$ that maximizes this objective:
\begin{align*}
    \max_{\pi} \sup_{\lVert f \rVert_L \le 1} J(\pi, f).
\end{align*}
Let $(\pi^*, f^*)$ be any global maximizer of this problem. By the assumptions of Theorem \ref{thm:eq-opt}, the tuple $(\mathcal{S}, \mathcal{A}, \mathcal{Z})$ is equipped with a measure-preserving group action, and the environment dynamics satisfy equivariance. 

We proceed by showing that the orbit of this optimal solution under the group action implies the existence of a symmetric optimizer.\\

\textbf{1. Rotating an Optimal Pair.}\quad 
For any $g\in G$, define the rotated policy $\pi_g^*$ and rotated function $f_g^*$ by:
\begin{align*}
\pi_g^*(a\mid s,z) &:=\pi^*(g^{-1} a \mid g^{-1}s, g^{-1}z), \\
f_g^*(s, z) &:= f^*(g^{-1} s, g^{-1} z).
\end{align*}
Because the metric $d$ is group-invariant, $\lVert f_g^* \rVert_L = \lVert f^* \rVert_L \le 1$, so $f_g^*$ remains feasible.
The joint distribution induced by the rotated policy rotates covariantly: $p_{\pi^*_g}(s,z)= p_{\pi^*}(g^{-1}s, g^{-1}z)$ (rotated by left regular representation, Appendix~\ref{app:left-regular}). 
Using the measure-preserving property of the group action, we compute the objective value:
\begin{align*}
&J(\pi_g^*, f_g^*) \\
&= \E_{\substack{z \sim p(z) \\ s \sim p_{\pi_g^*}(s \mid z)}}[ f_g^*(s,z) ] - \E_{\substack{z \sim p(z) \\ s \sim p_{\pi_g^*}(s)}}[ f_g^*(s,z) ] \\ 
&= \E_{(s,z) \sim p_{\pi_g^*}}[ f^*(g^{-1}s,g^{-1}z) ] \\
&\qquad\qquad\qquad\qquad\qquad- \E_{(s,z) \sim p_{\pi_g^*}(s)p(z)}[ f^*(g^{-1}s,g^{-1}z) ] \\ 
&=\E_{(s',z') \sim p_{\pi^*}}[ f^*(s',z') ] - \E_{(s',z') \sim p_{\pi^*}(s')p(z')}[ f^*(s',z') ] \\ 
&= J(\pi^*, f^*).
\end{align*}
where $s'=g^{-1}s, z'=g^{-1}z$ are change of variables.
Thus, for every $g\in G$, the pair $(\pi_g^*, f_g^*)$ is also a global maximizer.\\

\textbf{2. Symmetrizing the Policy.}\quad
Since averaging policies directly is non-trivial, we work in the space of occupancy measures. Let $\rho^{\pi}$ denote the state-action-skill occupancy measure induced by $\pi$. The set of valid occupancy measures is convex, and the WDM objective is linear with respect to $\rho^{\pi}$ (for a fixed optimal $f$, the maximum value is linear in $\rho$).
We define the symmetrized occupancy measure $\bar{\rho}^*$ by averaging over the Haar measure of the compact group $G$:
\begin{equation*}
\bar \rho^*(s,a,z) := \int_{G} \rho^{\pi_g^*}(s,a,z) ,d\mu(g).
\end{equation*}
Since every $\rho^{\pi_g^*}$ is optimal, by convexity, $\bar \rho^*$ is also optimal. We verify that $\bar \rho^*$ is group-invariant. 
Using $\rho^{\pi_g^*}(s,a,z) = \rho^{\pi^*}(g^{-1}s,g^{-1}a,g^{-1}z)$:
\begin{align*}
\bar \rho^*(gs,ga,gz)&= \int_{G} \rho^{\pi^*}(h^{-1}gs,h^{-1}ga,h^{-1}gz) ,d\mu(h) \\
&= \int_{G} \rho^{\pi^*}(k^{-1}s,k^{-1}a,k^{-1}z) ,d\mu(k) \quad  \\
&\qquad\qquad\qquad\qquad\qquad\qquad(\text{subst.}~ k=h^{-1}g) \\
&= \bar \rho^*(s,a,z).
\end{align*}
We define the equivariant policy $\bar \pi^*$ induced by $\bar \rho^*$ as $\bar \pi^*(a \mid s,z) \propto \bar \rho^*(s,a,z)$. Since $\bar \rho^*$ is optimal, $\bar \pi^*$ is an optimal equivariant policy.\\

\textbf{3. Symmetrizing the Scoring Function.}\quad 
Since $\bar{\pi}^*$ is optimal, it maximizes the WDM objective. Let $\tilde{f}$ be any optimal scoring function for $\bar{\pi}^*$, i.e., $J(\bar{\pi}^*, \tilde{f}) = \wdm(\bar{\pi}^*)$. Since $\bar{\pi}^*$ is equivariant, the objective $J(\bar{\pi}^*, \cdot)$ is invariant under the rotation of $f$. Thus, $J(\bar{\pi}^*, \tilde{f}_g) = J(\bar{\pi}^*, \tilde{f})$ for all $g$. We define the group-averaged scoring function:
\begin{equation*}
    \bar f^*(s,z) := \int_{G} \tilde{f}_g(s,z), d\mu(g).
\end{equation*}
By construction, $\bar f^*$ is group-invariant and $1$-Lipschitz (by convexity of the Lipschitz ball). By the linearity of $J$ with respect to $f$:
\begin{align*}
J(\bar{\pi}^*, \bar f^*) = \int_{G} J(\bar{\pi}^*, \tilde{f}_g) , d\mu(g) = \wdm(\bar{\pi}^*).
\end{align*}
\\
\textbf{Conclusion.}\quad
The pair $(\bar{\pi}^*, \bar{f}^*)$ constructed above satisfies:\begin{enumerate}\item \textbf{Global Optimality:} $J(\bar{\pi}^*, \bar{f}^*) = J(\pi^*, f^*)$.\item \textbf{Symmetry:} $\bar{\pi}^*$ is equivariant and $\bar{f}^*$ is group-invariant.
\end{enumerate}
\end{proof}

\subsection{Proof of Proposition~\ref{prop:gwdm}}
\label{app:gwdm}
\theoremRestate{thm:eq-opt}
\begin{propT}
    If the distance metric $d$ is group-invariant, then the
    group-invariant WDM satisfies
    \[
    I_{\mathcal W}^{G}(g\mathcal S;gZ) = I_{\mathcal W}^{G}(\mathcal S;Z)
    \]
    for all $g\in G$.
\end{propT}

\begin{proof}
We analyze the objective $I_{\mathcal W}^{G}(\mathcal S;Z)$ defined as the supremum over the function class $\mathcal{F}_{G}$ (Definition~\ref{def:gwdm}).
Consider the objective evaluated on the jointly transformed variables $(g\mathcal S, gZ)$:
\begin{align*}
&I_{\mathcal W}^{G}(g\mathcal S; gZ)\\
&= \sup_{f \in \mathcal{F}_{G}} \left (
\mathbb{E}_{p(g^{-1}s,g^{-1}z)}[f(s,z)]
- \mathbb{E}_{p(g^{-1}s)p(g^{-1}z)}[f(s,z)] \right )
\end{align*}
where the expectation is taken with respect to the density $L_g p$, i.e., $(L_g p)(s,z):=p(g^{-1}s,g^{-1}z)$ is the left regular representation of $G$ on densities (see Appendix.~\ref{app:left-regular}).
For the first expectation, we write:
\begin{equation*}
    \mathbb{E}_{p(g^{-1}s,g^{-1}z)}[f(s,z)]
    = \int_{\mathcal S\times Z} f(s,z)\, p(g^{-1}s,g^{-1}z)\, ds\,dz.
\end{equation*}
We perform the change of variables $u = g^{-1}s$, $v = g^{-1}z$ (so $s = gu$, $z = gv$). Since the action is volume-preserving ($ds\,dz = du\,dv$):
\begin{align*}
    \int_{\mathcal S\times Z} f(s,z)\, &p(g^{-1}s,g^{-1}z)\, ds\,dz \\
    &= \int_{\mathcal S\times Z} f(gu,gv)\, p(u,v)\, du\,dv.
\end{align*}

{Crucially, since $f \in \mathcal{F}_{G}$, we have $f(gu, gv) = f(u, v)$ by definition.} Thus:
\begin{align*}
    &= \int_{\mathcal S\times Z} f(u,v)\, p(u,v)\, du\,dv \\
    &= \mathbb{E}_{p(u,v)}[f(u,v)].
\end{align*}
An analogous calculation for the product of marginals:
\begin{align*}
\mathbb{E}_{p(g^{-1}s)p(g^{-1}z)}[f(s,z)]
&= \int f(s,z)\,p(g^{-1}s)\,p(g^{-1}z)\,ds\,dz \\
&= \int f(gu,gv)\,p(u)\,p(v)\,du\,dv \\
&= \mathbb{E}_{p(u)p(v)}[f(gu,gv)].
\end{align*}

Substituting these back into the supremum:
\begin{align*}
   I_{\mathcal W}^{G}&(g\mathcal S; gZ) \\
   &= \sup_{f \in \mathcal{F}_{G}}
      \Big( \mathbb{E}_{p(s,z)}[f(s,z)] - \mathbb{E}_{p(s)p(z)}[f(s,z)] \Big) \\
   &= I_{\mathcal W}^{G}(\mathcal S;Z).
\end{align*}
\end{proof}

\subsection{Proof of Proposition~\ref{prop:group-avg}}
\label{app:group-avg}

\propRestate{prop:group-avg}

\begin{propT}[Properties of the group-averaged scoring function]
Let $f:\mathcal S\times Z\to \mathbb R$ be $1$-Lipschitz with respect to
a group-invariant metric $d$ on $\mathcal S\times Z$. Then its group
average $\tilde f$ is group-invariant and $1$-Lipschitz with respect to
$d$. In particular, $\tilde f \in \mathcal{F}_{G}$.
\end{propT}

\begin{proof}

\textbf{Group-invariance of $\tilde{f}$.} \quad
Define the group-averaged function as $\tilde{f}(s,z)=\int_{G} f(gs,gz)d\mu(g)$, where $\mu$ is the Haar measure on the compact group $G$. For any $h\in G$,
\begin{align*}
    \tilde{f}(hs,hz)&=\int_G f(hgs,hgz)d\mu(g) \\
    &= \int_G f(g's,g'z) d\mu(g') ~~(g'=gh)\\
    &= \tilde{f}(s,z).
\end{align*}
Thus, $\tilde{f}$ is group-invariant.\\ 

\textbf{Preservation of the 1-Lipschitz condition.}\quad
Assume that $f$ is 1-Lipschitz with respect to a group-invariant metric $d$, so that $d((gs_1,gz_1),(gs_2,gz_2))=d((s_1, z_1), (s_2, z_2))$. Then for any $(s_1, z_1)$ and $(s_2, z_2)$,
\begin{align*}
    &\|\tilde{f}(s_1,z_1)-\tilde{f}(s_2, z_2)\| \\
    &= \left \|\int_G f(gs_1, gz_1)-f(gs_2, gz_2) d\mu(g) \right  \| \\
    &\leq \int_G \|f(gs_1, gz_1)-f(gs_2, gz_2)\| d\mu(g)\\
    &\quad\quad\quad\quad\quad\quad\quad\quad\quad\quad\quad\quad\quad\quad(\text{Jensen's inequality}) \\
    &\leq \int_G d((gs_1, gz_1), (gs_2, gz_2))d\mu(g) \\
    &\quad\quad\quad\quad\quad\quad\quad\quad\quad\quad\quad\quad(\text{1-Lipschitz property of } f ) \\
    &=\int_G d((s_1, z_1), (s_2, z_2)) ~d\mu(g)\\
    &\quad\quad\quad\quad\quad\quad\quad\quad\quad\quad\quad\quad\quad\quad(d\text{ is group-invariant}) \\
    &= d((s_1 , z_1), (s_2, z_2))
\end{align*}
Since $\mu$ is a probability measure, $\tilde{f}$ is 1-Lipschitz. \\
\end{proof}

\subsection{Invariance of the State-Skill Occupancy Measure}
\label{app:occ-invariance}

\begin{lemma}[Invariance of the State-Skill Distribution]
\label{lem:occ-invariance}
Assume the environment dynamics are group-invariant such that $P(gs'|gs, ga) = P(s'|s,a)$. Furthermore, assume the initial state distribution $p_0(s)$ and skill prior $p(z)$ are group-invariant. If the skill-conditioned policy is equivariant, i.e., $\pi(ga|gs, gz) = \pi(a|s,z)$, then the induced joint distribution at any timestep $t$ is group-invariant:
\begin{equation*}
p_t(gs, gz) = p_t(s, z), \quad \forall g \in G.
\end{equation*}
Consequently, the state-skill occupancy measure $\rho^\pi(s,z) = \frac{1}{T}\sum_{t=0}^{T-1} p_t(s,z)$ is also group-invariant.
\end{lemma}

\begin{proof}
We prove the lemma through mathematical induction.

\textbf{Base Case ($t=0$):}\\
The joint distribution at $t=0$ can be determined by the independent priors:
\[
p_0(s,z) = p_0(s)p(z).
\]
When applied by the group transformation $g\in G$,
\[
p_0(gs,gz) = p_0(gs)p(gz) = p_0(s)p(z)=p_0(s,z).
\]
\\
\textbf{Inductive step:}
Assume the hypothesis holds for time $t$, i.e., $p_t(gs,gz)=p_t(s,z)$. 
The marginal transition probability from $(s,z)$ to $s'$ under policy $\pi$ is:
\[
\mathcal T(s'\mid s,z) = \int_{\mathcal A}P(s'\mid s,a) \pi(a\mid s,z)\, da.
\]
First, we show that this probability is group-invariant. Consider the transition from rotated inputs $gs,gz$ to $gs'$:
\begin{align*}
    \mathcal{T}(gs' \mid gs, gz) &= \int_{\mathcal{A}} P(gs' \mid gs, a) \pi(a \mid gs, gz) \, da. \\
    &= \int_{\mathcal A}P(gs'\mid gs, g\bar a) \pi(g\bar a\mid gs, gz)\, d\bar a \\
    & \quad\quad\quad\quad\quad\quad\quad(\text{change of variable}~~a=g\bar a)\\
    &=\int_{\mathcal A} P(s'\mid s,\bar a)\pi(\bar a\mid s,z)\, d\bar a \\
    &=\mathcal T(s'\mid s,z).
\end{align*}
where we use the change of variable $a = g\bar a$, and $g\in G$ is a volume-preserving group transform, $da = d\bar a$
Now, we express $p_{t+1}(s',z)$ by marginalizing over the previous state $s$:
\[
p_{t+1}(s',z) = \int_{\mathcal S} \mathcal T(s'\mid s,z)p_t (s,z)\, ds.
\]
At transformed pair $(gs',gz)$:
\[
p_{t+1}(gs',gz) = \int_{\mathcal S} \mathcal T(gs'\mid s,gz)p_t(s,gz)\, ds.
\]
We again use the change of variable: $s= g\bar s$, $ds = d\bar s$:
\begin{align*}
p_{t+1}(gs',gz) &= \int_{\mathcal S} \mathcal T(gs'\mid g\bar s, gz)p_t(g\bar s, gz)\, d\bar s \\
&=\int_{\mathcal S} \mathcal T(s'\mid \bar s, z)p_t (\bar s, z)\, d\bar s \\
&=p_{t+1}(s',z).
\end{align*}
By induction, $p_t(gs,gz)=p_t(s,z)$ for all $t$. 
Therefore, $\rho^\pi(s,z)$ is also group-invariant.
\end{proof}

\subsection{Group-invariant Temporal Distance under Group-Invariant MDPs}
\label{gtd}
The temporal distance $d_T(s_1,s_2)$ between two states $s_1$ and $s_2$ is defined as the expected number of transitions needed to reach state $s_2$ from state $s_1$ under policy $\pi$ \cite{durugkar2021adversarial}. It satisfies the following recursive relationship: 
\begin{equation*}
d_T(s_1, s_2) =
\begin{cases}
0 & \text{if } s_1 = s_2,\\[4pt]
1 + \mathbb{E}_{a, s' \mid s_1}\big[d_T(s', s_2)\big] & \text{if } s_1 \neq s_2.
\end{cases}
\end{equation*}
For the base case, if $s_1=s_2$, it leads to $d_T(gs_1, gs_2)=0$, since the group action is assumed to be bijective, $gs_1=gs_2$. So, the group-invariance is satisfied in the base case:
\begin{equation*}
    d_T(gs_1, gs_2)=d_T(s_1,s_2)=0
\end{equation*}
For the inductive step, assume that the group-invariance holds for states closer to $s_2$, $d_T(gs',gs_2)=d_T(s',s_2)$. Then from the recursive term,
\begin{equation*}
    d_T(gs_1,gs_2)=1+\mathbb{E}_{a\sim \pi(\cdot|gs_1)}\mathbb{E}_{s'\sim P(\cdot|gs_1,a)}[d_T(s', gs_2)]
\end{equation*}
By the Lemma from \cite{wang2022equivariant} that $gS=S$ and $gA=A$, we can set $a=g\tilde a$ and $s'=gs''$. Then, in the group-invariant MDPs, $\pi(g\tilde a|gs_1)=\pi(\tilde a|s_1)$ and $P(gs''|gs_1, g\tilde a)=P(s''|s_1, \tilde a)$. Substituting these into the equation,
\begin{equation*}
    d_T(gs_1,gs_2)=1+\mathbb{E}_{\tilde a\sim \pi(\cdot|s_1)}\mathbb{E}_{s''\sim P(\cdot|s_1,\tilde{a})}[d_T(gs'', gs_2)]
\end{equation*}
By the inductive hypothesis, $d_T(gs'',gs_2)=d_T(s'',s_2)$. Therefore,
\begin{align*}
    d_T(gs_1,gs_2)&=1+\mathbb{E}_{\tilde a\sim \pi(\cdot|s_1)}\mathbb{E}_{s''\sim P(\cdot|s_1,\tilde{a})}[d_T(s'', s_2)] \\
    &=1+\mathbb{E}_{a\sim \pi(\cdot|s_1)}\mathbb{E}_{s'\sim P(\cdot|s_1,a)}[d_T(s', s_2)]
\end{align*}
where its equation is equivalent to the recursive definition of $d_T(s_1, s_2)$. Thus, by the mathematical induction, temporal distance is group-invariant under the group-invariant MDPs, $d_T(gs_1,gs_2)=d_T(s_1, s_2)$. \\

\subsection{Proof of Theorem~\ref{theorem:semiMDP}}
\label{gsmdp}

In this section, we prove that if the underlying MDP is group-invariant and the low-level skill policy is equivariant, the induced state transition probability in the high-level Semi-MDP (defined over a fixed interval $k$) retains group invariance.

Let $G$ be a compact group acting on the state space $\mathcal S$ and action space $\mathcal A$. We denote the transition probability of the underlying MDP as $P^l(s' \mid s, a)$. Group-invariance of the dynamics implies that $P^l(gs' \mid gs, ga) = P^l(s' \mid s, a)$ for all $g \in G$. Furthermore, let $\pi(a \mid s, z)$ be the low-level skill policy. Equivariance of the policy implies $\pi(ga \mid gs, gz) = \pi(a \mid s, z)$. We define $P_k(s_{t+k} \mid s_t, z)$ as the transition probability of the Semi-MDP, representing the likelihood of reaching state $s_{t+k}$ from $s_t$ after executing skill $z$ for $k$ steps.

\theoremRestate{theorem:semiMDP} 
\begin{theoremR}[Group-invariant Fixed-interval Semi-MDP] 
Assume the environment dynamics are group-invariant and the pretrained skill policy is equivariant: \begin{align*} 
P^l(gs' \mid gs, ga) &= P^l(s' \mid s, a) \\ 
\pi^l(ga \mid gs, gz) &= \pi^l(a \mid s, z) \quad \forall g \in G. 
\end{align*} 
Then, the induced high-level transition kernel (Eq.~(\ref{eq:Pk-def})) satisfies group invariance: \begin{equation*} 
P_k(g s_{t+k} \mid g s_t, g z) = P_k(s_{t+k} \mid s_t, z), \qquad \forall g \in G. 
\end{equation*} 
Consequently, the fixed-interval semi-MDP is group-invariant whenever the high-level reward is group-invariant. 
\end{theoremR}

\begin{proof} The probability of reaching state $s_{t+k}$ from $s_t$ under skill $z$ is obtained by marginalizing over all possible intermediate trajectories of actions $a_{t:t+k-1}$ and states $s_{t+1:t+k-1}$. Let $\tau_{int}$ denote the set of these intermediate variables. We write the transition density as:
\begin{align*}
    &P_k(s_{t+k} \mid s_t, z) \\
    &= \int \dots \int \prod_{i=0}^{k-1} \bigg[ 
        \pi(a_{t+i} \mid s_{t+i}, z) \cdot \\
        &\quad\quad\quad\quad\quad\quad\quad\quad\quad\quad P^l (s_{t+i+1} \mid s_{t+i}, a_{t+i}) 
    \bigg] \, d\tau_{int}.
\end{align*}

Now, consider the transition probability under the transformed inputs $gs_t$ and $gz$. Let the integration variables representing the intermediate trajectory be denoted by the primed variables $s'_{t+i}$ and $a'_{t+i}$:
\begin{align*}
    &P_k (g s_{t+k} \mid g s_t, gz) \\
    &= \int \dots \int \prod_{i=0}^{k-1} \bigg[ 
        \pi(a'_{t+i} \mid s'_{t+i}, g z) \cdot \\
        &\quad\quad\quad\quad\quad\quad\quad\quad\quad\quad P^l(s'_{t+i+1} \mid s'_{t+i}, a'_{t+i}) 
    \bigg] \, d\tau'_{int}.
\end{align*}

We perform a change of variables for the entire trajectory. Let $s'_{t+i} = g\hat{s}_{t+i}$ and $a'_{t+i} = g\hat{a}_{t+i}$. Since the group action is volume-preserving (unit Jacobian determinant for compact groups), the differential element remains unchanged, $d\tau'_{int} = d\hat{\tau}_{int}$. Substituting these variables and applying the group-invariance of $P^l$ and the equivariance of $\pi$:
\begin{align*} 
    &P_k(g s_{t+k} \mid g s_t, g z) \\
    &= \int \dots \int \prod_{i=0}^{k-1} \bigg[ 
        P^l(g\hat{s}_{t+i+1} \mid g\hat{s}_{t+i}, g\hat{a}_{t+i}) \\
    &\quad \qquad \qquad\qquad\qquad\qquad\qquad \cdot \pi(g\hat{a}_{t+i} \mid g\hat{s}_{t+i}, gz) 
    \bigg] \, d\hat{\tau}_{int} \\ 
    &= \int \dots \int \prod_{i=0}^{k-1} \bigg[ 
        P^l(\hat{s}_{t+i+1} \mid \hat{s}_{t+i}, \hat{a}_{t+i}) \\
    &\quad \qquad \qquad \qquad\qquad\qquad\qquad\cdot \pi(\hat{a}_{t+i} \mid \hat{s}_{t+i}, z) 
    \bigg] \, d\hat{\tau}_{int} \\
    &= P_k(s_{t+k} \mid s_t, z).
\end{align*}
Thus, the Semi-MDP transition dynamics are group-invariant.
\end{proof}

\section*{Appendix B. Representation Theory Background}
\label{background}

\subsection{Left Regular Representation}
\label{app:left-regular}

Let $G$ be a compact group acting on a space $\mathcal X$ (e.g., $\mathcal X=\mathcal S \times Z$). The \emph{left regular representation} on functions $f:\mathcal X\to \mathbb R$ is defined as:
\[
(L_g f)(x):=f(g^{-1}x),\quad g\in G
\]
If $p(x)$ is a probability density on $\mathcal X$ with respect to group-invariant measure, the corresponding pushforward density transforms identically to the function \cite{folland2016course}:
\[
(L_g p)(x):=p(g^{-1} x).
\]
We use this convention throughout the main text (e.g., $L_g p(s,z)=p(g^{-1}s, g^{-1}z)$

\subsection{Compact Groups and Haar Measure}
\label{app:haar}
Throughout this work, we consider a compact Hausdorff topological group $G$, i.e., 
a group equipped with a Hausdorff topology such that the multiplication $g.h\mapsto gh$ and inversion $g\mapsto g^{-1}$ are continuous and the underlying topological space is compact \cite{folland2016course}.

On any locally compact group, there exists a left-invariant measure $\mu$, called a (left) Haar measure, satisfying
\[
\mu(gS)=\mu(S),\quad \forall g\in G, S\subseteq G.
\]
Such a left Haar measure is unique up to multiplication by a positive constant. 
Since $G$ is compact, we can normalize $\mu(G)=1$, so that $\int_G h(g)~d\mu(g)$ is the average of a function $h$ over the group. 
In particular, the expressions of $\int_G f(gs,gz)\ d\mu(g)$ in the main text are precisely such group averages.

\subsection{Group Fourier Transform for $SO(2)$}\label{fourier}
In our experiments we work with a finite cyclic subgroup
$G = C_N \subset SO(2)$, so we use the discrete form of the group
Fourier transform. For a finite compact group $G$, the Peter--Weyl
theorem \cite{ceccherini2008harmonic} states that the matrix entries of
all irreducible representations $\rho_j \in \hat G$ form a complete
orthonormal basis of $L^2(G)$ (with respect to the normalized counting
measure). Consequently, any $f \in L^2(G)$ admits an expansion of the
form
\begin{align*}
f(g)
&= \sum_{\rho_j \in \hat G} \sum_{m,n < d_j}
    w_{j,m,n}\,\sqrt{d_j}\,[\rho_j(g)]_{mn} \\
&= \sum_{\rho_j \in \hat G}
    \sqrt{d_j}\,\mathrm{Tr}\!\big(\rho_j(g)^\top \hat f(\rho_j)\big),
\end{align*}
where $d_j$ is the dimension of $\rho_j$, indices $m,n$ range over the
matrix entries, and $\hat f(\rho_j)\in \mathbb R^{d_j\times d_j}$ is the
matrix collecting the coefficients $w_{j,m,n} \in \mathbb R$.

For each basis element $[\rho_j(g)]_{mn}$, the associated coefficient is
obtained by projecting $f$ onto that basis function:
\begin{equation*}
w_{j,m,n}
= \frac{1}{|G|}
  \sum_{g \in G} f(g)\,\sqrt{d_j}\,[\rho_j(g)]_{mn},
\end{equation*}
where $|G|$ denotes the order of the finite group $G$. Stacking these
coefficients into a matrix yields
\begin{equation*}
\hat f(\rho_j)
= \frac{1}{|G|}
  \sum_{g\in G} f(g)\,\sqrt{d_j}\,\rho_j(g),
\end{equation*}
which we refer to as the (group) Fourier transform of $f$ at
representation $\rho_j$.

In the main text, the ``spectral'' or ``Fourier'' feature maps
$\phi_F(s)$ and $\psi_F(z)$ are precisely parametrizations of such
coefficient vectors across the irreducible representations of $G$.

\bibliographystyle{IEEEtran}
\bibliography{references}


 
\vspace{11pt}

\vfill

\end{document}